\documentclass[conference]{IEEEtran}
\IEEEoverridecommandlockouts
\usepackage[utf8]{inputenc} 
\usepackage[T1]{fontenc}    
\usepackage{booktabs}       
\usepackage{multicol}
\usepackage{multirow}
\usepackage{nicefrac}       
\usepackage{microtype}      
\usepackage{xcolor}         
\usepackage{graphicx}
\usepackage{wrapfig}
\usepackage[misc]{ifsym}
\usepackage{amssymb}
\usepackage{amsmath,amsfonts,mathtools}
\usepackage{amsthm}         
\usepackage{pifont}
\usepackage[neveradjust]{paralist}
\usepackage{adjustbox}
\usepackage{natbib}         

\usepackage{tikz}
\usetikzlibrary{positioning, arrows.meta, calc}

\usepackage{algorithm}
\usepackage{algpseudocode}

\usepackage{enumitem}
\theoremstyle{plain} 
\newtheorem{lemma}{Lemma}
\newtheorem{proposition}{Proposition}

\theoremstyle{definition} 
\newtheorem{remark}{Remark}

\PassOptionsToPackage{
  colorlinks=true,
  linkcolor=blue!60!black,
  citecolor=blue!60!black,
  urlcolor=blue!60!black
}{hyperref}
\usepackage{hyperref}      

\usepackage[hyperfirst=false,acronym]{glossaries}
\glsdisablehyper
\robustify{\gls}
\robustify{\Gls}
\newacronym{dag}{DAG}{directed acyclic graph}
\newacronym{cpdag}{CPDAG}{completed partially directed acyclic graph}
\newacronym{mag}{MAG}{maximal ancestral graph}
\newacronym{pag}{PAG}{partial ancestral graph}
\newacronym{pc}{PC}{Peter-Clark}
\newacronym{mpdag}{MPDAG}{maximally partially directed acyclic graph}
\newacronym{ci}{CI}{conditional independence}
\newacronym{sem}{SEM}{structural equation modeling}

\usepackage[noabbrev,nameinlink,capitalize]{cleveref}
\usepackage{threeparttable}

\graphicspath{%
  {figures/}%
}

\newcommand{\Pa}{\textit{\textbf{Pa}}}
\newcommand{\An}{\textit{\textbf{An}}}
\newcommand{\De}{\textit{\textbf{De}}}
\newcommand{\Ch}{\textit{\textbf{Ch}}}
\newcommand{\Sib}{\textit{\textbf{Sib}}}
\newcommand{\doop}[1]{\textit{\textbf{do}}(#1)}

\newcommand{\method}{\emph{b-LOAD}}

\newcommand{\cmark}{\ding{51}}
\newcommand{\xmark}{\ding{55}}
\newcommand{\myparagraph}[1]{\noindent\textbf{#1}}

\usepackage{graphicx}
\graphicspath{%
  {figures/}%
  {new_figures/}%
  {new_figures/synthetic_experiments/}%
  {new_figures/real_data/sachs/}%
  {new_figures/real_data/dream4/}%
  {new_figures/noised_experiments/}%
}
\usepackage{tikz}
\usetikzlibrary{
  arrows.meta, positioning, calc, shapes.geometric, shapes.misc,
  backgrounds, fit, decorations.pathreplacing, decorations.pathmorphing,
  patterns, shadings, shadows.blur
}
\usepackage{pgfplots}
\pgfplotsset{compat=1.17}

\usepackage{booktabs}
\usepackage{multirow}
\usepackage{array}
\usepackage{caption}
\usepackage{subcaption}
\usepackage{enumitem}

\usepackage{xcolor}
\usepackage{tcolorbox}
\tcbuselibrary{skins,breakable}

\definecolor{bloadBlue}{HTML}{08519c}
\definecolor{bloadDeep}{HTML}{08306b}
\definecolor{bloadMid}{HTML}{3182bd}
\definecolor{bloadLight}{HTML}{9ecae1}
\definecolor{loadGray}{HTML}{7F7F7F}
\definecolor{successGreen}{HTML}{2CA02C}
\definecolor{failRed}{HTML}{D62728}
\definecolor{accentOrange}{HTML}{E69F00}
\definecolor{softYellow}{HTML}{fee0b6}
\definecolor{paleGrey}{HTML}{F0F0F0}

\definecolor{bloadblue}{RGB}{55,105,170}
\definecolor{bloadgreen}{RGB}{40,135,95}
\definecolor{bloadorange}{RGB}{210,125,45}
\definecolor{bloadred}{RGB}{190,65,65}
\definecolor{bloadgray}{RGB}{90,95,105}
\definecolor{bloadlight}{RGB}{245,247,250}
\definecolor{bloadpurple}{RGB}{120,80,170}

\definecolor{failRed}{rgb}{0.72,0.08,0.08}
\definecolor{bloadBlue}{rgb}{0.00,0.28,0.60}
\definecolor{bloadMid}{rgb}{0.05,0.42,0.75}
\definecolor{bloadDeep}{rgb}{0.00,0.18,0.40}
\definecolor{successGreen}{rgb}{0.16,0.55,0.22}
\definecolor{paleGrey}{rgb}{0.94,0.94,0.94}
\definecolor{loadGray}{rgb}{0.42,0.42,0.42}

\begin{document}

\sloppy

\title{Knowledge-Informed Local Causal Discovery of Optimal Adjustment Sets}

\author{
\IEEEauthorblockN{
Seong Woo Ahn\IEEEauthorrefmark{1}\thanks{Corresponding author: seong-woo.ahn@universite-paris-saclay.fr},
Alessandro Leite\IEEEauthorrefmark{2},
Jos\'e Lucas De Melo Costa\IEEEauthorrefmark{1}\\
Fabrice Popineau\IEEEauthorrefmark{1},
Bich-Li\^en Doan\IEEEauthorrefmark{1},
Arpad Rimmel\IEEEauthorrefmark{1}
}
\IEEEauthorblockA{
\IEEEauthorrefmark{1}Laboratoire Interdisciplinaire des Sciences du Num\'erique,
Universit\'e Paris-Saclay, CNRS, CentraleSup\'elec, France}
\IEEEauthorblockA{
\IEEEauthorrefmark{2}LITIS UR 4108,
INSA Rouen Normandy, University of Rouen Normandy, France}
}

\maketitle

\begin{abstract}
Local causal discovery is a scalable alternative to global structure learning. However, it can struggle to identify valid adjustment sets in data-scarce settings because of finite-sample uncertainty, incomplete local neighborhoods, and unresolved Markov equivalence. Although many application domains provide structured background knowledge, its integration into local causal discovery remains limited. We propose \method, a knowledge-informed extension of the LOAD~\cite{schubert2025local} algorithm for local discovery of optimal adjustment sets. \method\ incorporates prior edge constraints directly into the local structure-learning procedure and uses Meek's rules to expand the discovery frontier dynamically, yielding a knowledge-constrained~\gls{mpdag} over the relevant local subgraph. This strategy helps prevent structurally relevant nodes introduced by prior knowledge from being excluded by local search. We prove that, under sound background knowledge, the procedure monotonically refines the admissible equivalence class and can enlarge the set of identifiable causal queries, enabling recovery of optimal adjustment sets that are not identifiable from observational conditional-independence information alone. Empirically, \method\ improves downstream causal effect estimation relative to purely data-driven and standard knowledge-augmented baselines, particularly in data-scarce and structurally complex regimes. Results on real-world biological networks show that locally targeted prior knowledge provides the largest gains and remains beneficial under moderate structural noise. These findings position \method\ as a scalable approach for converting fragmented domain knowledge into more reliable causal-effect estimation.
\end{abstract}


\section{Introduction}\label{sec:introduction}
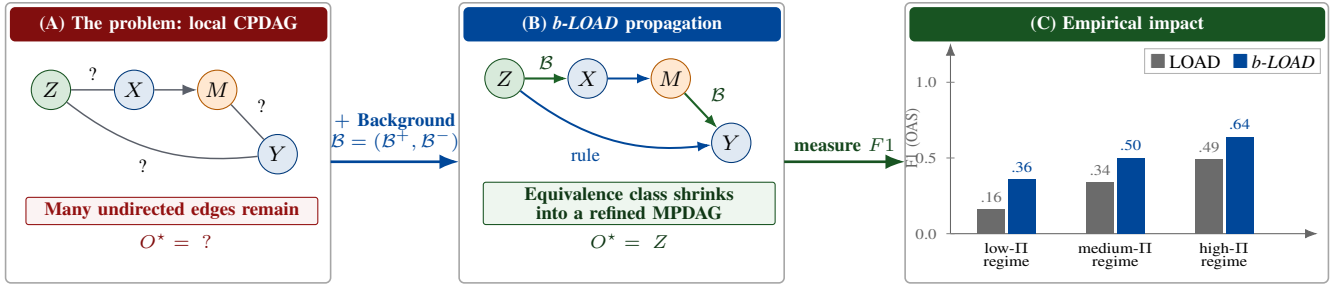
\begin{figure*}[t]
\centering
\vspace{-0.6em}

\begin{adjustbox}{max width=\textwidth}
\begin{tikzpicture}[
  every node/.style={font=\footnotesize},
  x=1cm,
  y=1cm,
  >=Latex,
  panel-box/.style={
    draw=black!35,
    rounded corners=2pt,
    line width=0.6pt,
    fill=white
  },
  title-box/.style={
    rounded corners=2pt,
    text=white,
    font=\scriptsize\bfseries,
    minimum height=0.45cm,
    align=center
  },
  stamp-fail/.style={
    draw=failRed!80!black,
    fill=failRed!5,
    text=failRed!80!black,
    font=\scriptsize\bfseries,
    rounded corners=1pt,
    inner xsep=3pt,
    inner ysep=2pt,
    align=center
  },
  stamp-success/.style={
    draw=successGreen!60!black,
    fill=successGreen!8,
    text=successGreen!40!black,
    font=\scriptsize\bfseries,
    rounded corners=1pt,
    inner xsep=3pt,
    inner ysep=2pt,
    align=center
  },
  node-causal/.style={
    circle,
    draw=black!55,
    fill=white,
    minimum size=0.52cm,
    inner sep=0pt
  },
  node-target/.style={
    circle,
    draw=bloadBlue!80!black,
    fill=bloadBlue!12,
    minimum size=0.52cm,
    inner sep=0pt
  },
  node-osel/.style={
    circle,
    draw=successGreen!60!black,
    fill=successGreen!18,
    minimum size=0.52cm,
    inner sep=0pt
  },
  node-mediator/.style={
    circle,
    draw=orange!70!black,
    fill=orange!18,
    minimum size=0.52cm,
    inner sep=0pt
  },
  edge-undirected/.style={
    draw=bloadgray,
    line width=0.55pt
  },
  edge-required/.style={
    draw=successGreen!70!black,
    -{Latex[length=1.8mm]},
    line width=0.8pt
  },
  edge-meek/.style={
    draw=bloadBlue,
    -{Latex[length=1.8mm]},
    line width=0.75pt
  },
  edge-inferred/.style={
    draw=bloadMid,
    -{Latex[length=1.8mm]},
    line width=0.75pt
  },
  big-arrow/.style={
    -{Latex[length=2.6mm]},
    line width=1.0pt
  }
]

\begin{scope}[local bounding box=panelA]
  \draw[panel-box] (-0.3, 0.55) rectangle (4,4.25);

  \node[
    title-box,
    fill=failRed!70!black,
    anchor=north west,
    minimum width=4.2cm
  ] at (-0.25,4.20) {(A)\,\,The problem: local CPDAG};

  \node[node-osel]     (Z1) at (0.3,3.1) {$Z$};
  \node[node-target]   (X1) at (1.4,3.1) {$X$};
  \node[node-mediator] (M1) at (2.5,3.1) {$M$};
  \node[node-target]   (Y1) at (3.3,2.25) {$Y$};

  \draw[edge-undirected] (Z1) -- node[above,font=\scriptsize] {?} (X1);
  \draw[bloadgray,line width=0.55pt, ->] (X1) -- (M1);
  \draw[edge-undirected]
  (M1) -- node[midway, above right, font=\scriptsize] {?} (Y1);
  \draw[edge-undirected] (Z1) to[bend right=22]
    node[below left,font=\scriptsize] {?} (Y1);
    
  \node[
    stamp-fail,
    text width=3.65cm,
    anchor=center
  ] at (1.9,1.5)
  {Many undirected edges remain};

  \node[
    font=\scriptsize,
    anchor=center,
    text=failRed!70!black,
    align=center
  ] at (1.95,1.1)
  {$O^\star = \;\;?$};
\end{scope}

\draw[big-arrow, color=bloadBlue]
  (4,2.15) --
  node[
    midway,
    above,
    font=\scriptsize\bfseries,
    text=bloadBlue,
    align=center
  ]
  {$+$ Background \\ $\mathcal{B}=(\mathcal{B}^+,\mathcal{B}^-)$}
  (5.7,2.15);

\begin{scope}[xshift=6cm, local bounding box=panelB]
  \draw[panel-box] (-0.3, 0.55) rectangle (4,4.25);

  \node[
    title-box,
    fill=bloadBlue,
    anchor=north west,
    minimum width=4.2cm
  ] at (-0.25,4.20) {(B)\,\,\method\ propagation};

  \node[node-osel]     (Z2) at (0.3,3.25) {$Z$};
  \node[node-target]   (X2) at (1.4,3.25) {$X$};
  \node[node-mediator] (M2) at (2.5,3.25) {$M$};
  \node[node-target]   (Y2) at (3.3,2.4) {$Y$};

  \draw[edge-required] (Z2) -- node[above,font=\scriptsize,text=successGreen!45!black] {$\mathcal{B}$} (X2);
  \draw[edge-meek] (X2) -- (M2);
  \draw[edge-required] (M2) -- node[midway, above right, font=\scriptsize,text=successGreen!45!black] {$\mathcal{B}$} (Y2);
  \draw[edge-meek] (Z2) to[bend right=22] node[midway, below left, font=\scriptsize,text=bloadBlue] {rule} (Y2);

  \node[
    stamp-success,
    text width=3.65cm,
    anchor=center
  ] at (1.95,1.6)
  {Equivalence class shrinks\\[-1pt]
    into a refined MPDAG};

  \node[
    font=\scriptsize,
    anchor=center,
    text=successGreen!40!black,
    align=center
  ] at (1.95,1.1)
  {$O^\star = \; {Z}$};
\end{scope}

\draw[big-arrow, color=successGreen!60!black]
  (10,2.15) --
  node[
    midway,
    above,
    font=\scriptsize\bfseries,
    text=successGreen!50!black,
    align=center
  ]
  {measure $F1$}
  (11.6,2.15);

\begin{scope}[xshift=11.90cm, local bounding box=panelC]
  \draw[panel-box] (-0.3,0.55) rectangle (5.3,4.25);

  \node[
    title-box,
    fill=successGreen!60!black,
    anchor=north west,
    minimum width=5.5cm
  ] at (-0.25,4.20) {(C)\,\,Empirical impact};

  \begin{scope}[
      shift={(0.3cm,1.2cm)},
      xscale=0.9,
      yscale=0.8,
      transform shape
    ]

    \draw[->, color=black!60] (0,0) -- (0,3.15);
    \draw[->, color=black!60] (0,0) -- (5.40,0);

    \node[
      anchor=south,
      font=\scriptsize,
      rotate=90,
      text=black!60
    ] at (-0.32,1.58) {F1 (OAS)};

    \foreach \y/\lbl in {0/0.0,1.25/0.5,2.50/1.0} {
      \draw[color=black!50, line width=0.3pt]
        (-0.06,\y) -- (0.06,\y);
      \node[
        anchor=east,
        font=\scriptsize,
        text=black!60
      ] at (-0.08,\y) {\lbl};
    }

    \fill[loadGray]  (0.40,0) rectangle (0.80,0.16*2.5);
    \fill[bloadBlue] (0.85,0) rectangle (1.25,0.36*2.5);
    \node[anchor=south, font=\scriptsize, text=loadGray]  at (0.60,0.16*2.5) {.16};
    \node[anchor=south, font=\scriptsize, text=bloadBlue] at (1.05,0.36*2.5) {.36};
    \node[anchor=north, font=\scriptsize, align=center]   at (0.825,-0.02) {low-$\Pi$\\[-0.5ex]regime};

    \fill[loadGray]  (2.00,0) rectangle (2.40,0.34*2.5);
    \fill[bloadBlue] (2.45,0) rectangle (2.85,0.50*2.5);
    \node[anchor=south, font=\scriptsize, text=loadGray]  at (2.20,0.34*2.5) {.34};
    \node[anchor=south, font=\scriptsize, text=bloadBlue] at (2.65,0.50*2.5) {.50};
    \node[anchor=north, font=\scriptsize, align=center]   at (2.425,-0.02) {medium-$\Pi$\\[-0.5ex]regime};

    \fill[loadGray]  (3.60,0) rectangle (4.00,0.49*2.5);
    \fill[bloadBlue] (4.05,0) rectangle (4.45,0.64*2.5);
    \node[anchor=south, font=\scriptsize, text=loadGray]  at (3.80,0.49*2.5) {.49};
    \node[anchor=south, font=\scriptsize, text=bloadBlue] at (4.25,0.64*2.5) {.64};
    \node[anchor=north, font=\scriptsize, align=center]   at (4.025,-0.02) {high-$\Pi$\\[-0.5ex]regime};

  \end{scope}

  \node[
    anchor=north east,
    font=\scriptsize,
    fill=white,
    draw=black!30,
    inner sep=2pt
  ] at (5.2,3.65)
  {
    \textcolor{loadGray}{$\blacksquare$}\,LOAD\;
    \textcolor{bloadBlue}{$\blacksquare$}\,\method
  };
\end{scope}

\end{tikzpicture}
\end{adjustbox}

\caption{
Overview of \method.
(A) Without background knowledge, LOAD returns an unoriented local CPDAG, leaving the causal query potentially non-identifiable. 
(B) Background constraints $\mathcal{B}$ trigger \method\ propagation via Meek rules, refining the equivalence class into an MPDAG. 
(C) These orientations consistently improve optimal-adjustment-set recovery (F1) across the low-$\Pi$, medium-$\Pi$, and high-$\Pi$ graph complexity regimes.
}
\label{fig:bload-overview}

\vspace{-0.6em}
\end{figure*}
Modern AI systems excel at capturing statistical associations, yet many scientific and high-stakes applications require reasoning about the effects of actions. While predictive models exploit correlations, domains such as epidemiology, economics, and medicine demand answers to counterfactual questions, what would happen under some intervention, rather than what merely co-occurs~\citep{pearl2009causality,pearl2018book}. Causal inference provides the formal framework for this shift.

\Glspl{dag} provide a canonical representation of causal structure \citep{spirtes2000causation,pearl2009causality}. When the graph is known, causal effects can be identified using graphical criteria such as adjustment sets~\citep{pearl1993bayesian,perkovic2018complete} or the IDA estimator over a CPDAG's Markov equivalence class~\citep{maathuis2009estimating}. In practice, however, the causal graph must be inferred from observational data.

Causal structure learning methods recover equivalence classes of~\glspl{dag} from data, but face fundamental scalability limitations in high-dimensional settings. Local causal discovery mitigates this challenge by restricting attention to a target variable and its neighborhood, typically its Markov blanket \citep{pellet2008using,wang2014discovering,margaritis2000bayesian,tsamardinos2003iamb,aliferis2010local}. While this locality enables scalable inference, it introduces a key limitation: the resulting~\gls{cpdag} is often weakly oriented, especially in data-scarce or structurally complex regimes, preventing identification of causal effects.

\citet{schubert2025local} address this limitation with LOAD (Local Optimal Adjustments Discovery), which recovers optimal adjustment sets \citep{henckel2022graphical} by minimizing asymptotic estimation variance, from local structure alone. However, LOAD remains fundamentally data-driven. In low-sample regimes, the learned \gls{cpdag} typically retains many undirected edges, leaving causal effects non-identifiable. Moreover, LOAD's empirical evaluation is largely restricted to synthetic Erd\H{o}s--R\'enyi graphs at sample sizes that comfortably exceed the variable count---a setting at odds with the deployment scenarios that most motivate causal-effect estimation. In biomedical applications, gene-regulatory and protein-signaling networks routinely present hundreds of variables with comparable or fewer observations, inverting the standard $n \gg d$ assumption and driving \gls{ci}-based \glspl{cpdag} into the weakly-oriented regime where identification breaks down.

At the same time, many application domains combine limited observational data with rich sources of structured prior knowledge, such as curated knowledge graphs or domain constraints. These encode partial causal relationships and provide a natural source of orientation information. While the integration of background knowledge into causal graphs is well understood theoretically \citep{meek1995causal}, its use within local discovery pipelines for \emph{optimal adjustment set identification} remains largely unexplored.

We propose \textbf{\method}\footnote{Code: \url{https://anonymous.4open.science/r/b-LOAD/README.md}} (Background knowledge-augmented LOAD), a knowledge-informed framework for local causal discovery of optimal adjustment sets (\Cref{fig:bload-overview}). Our key design choice is to incorporate background knowledge \emph{directly into} the local structure-learning procedure rather than as a post-hoc refinement: required edge orientations seed the local graph at initialization, are re-imposed after every local update, and are propagated jointly with data-driven orientations via Meek's rules~\citep{meek1995causal} to yield a knowledge-constrained~\gls{mpdag} over the relevant local subgraph. This refinement strictly reduces the Markov equivalence class while preserving the robustness of constraint-based discovery, thereby improving the identifiability guarantees of the causal effects. The resulting \gls{mpdag} is then used to assess identifiability and extract the optimal adjustment set using local structural criteria.

We show theoretically that incorporating valid background knowledge monotonically improves identifiability and can strictly expand the set of queries for which optimal adjustment sets are recoverable. Empirically, \method\ consistently outperforms LOAD in data-scarce regimes with structurally complex causal mechanisms---precisely the conditions in which purely data-driven CI tests are unreliable and the learned CPDAG remains weakly oriented. In particular, on the Sachs protein signaling dataset~\cite{sachs2005causal}, LOAD fails to identify any valid adjustment set, whereas \method\ recovers identifiable effects when augmented with domain knowledge.

\myparagraph{Contributions.}
\begin{compactenum}[(1)]
    \item \textbf{Knowledge-informed local discovery with formal guarantees.} We introduce a principled framework for incorporating background knowledge into local causal discovery via~\gls{mpdag} refinement, proving that the procedure strictly refines the Markov equivalence class, monotonically improves identifiability, and can recover optimal adjustment sets not identifiable from data alone~(\Cref{sec:method}).
    
    \item \textbf{Empirical validation across synthetic and biological benchmarks.} On synthetic graphs spanning a wide statistical-power range and on two biological benchmarks (Sachs, DREAM4), knowledge-informed local discovery substantially improves adjustment-set identification, with the largest gains in data-scarce regimes~(\Cref{sec:empirical_evaluation}).
    
    \item \textbf{Prescriptive guidance for practitioners.} We characterize \emph{when} and \emph{how} prior knowledge is best supplied: globally distributed constraints maximize aggregate identification, while locally targeted constraints yield the most predictable causal-effect estimation, with concrete recommendations tied to the data regime and constraint noise level~(\Cref{sec:empirical_evaluation,sec:discussion}).
\end{compactenum}

\section{Preliminaries}\label{sec:preliminaries}

We introduce the key concepts and notation used throughout the paper, following the conventions of \citet{pearl2009causality} and
\citet{spirtes2000causation}.


A \textbf{causal graph} $\mathcal{G} \coloneqq (\mathcal{V}, \mathcal{E})$ encodes causal dependencies over observed variables 
$\mathcal{V} = \{X_1, \ldots, X_d\}$. A directed edge $X_i \to X_j$ 
represents a direct causal influence; $\Pa_{\mathcal{G}}(X_i)$, 
$\An_{\mathcal{G}}(X_i)$, $\De_{\mathcal{G}}(X_i)$, $\Ch_{\mathcal{G}}(X_i)$, and $\Sib_{\mathcal{G}}(X_i)$ denote the parents, ancestors, descendants, children, and siblings of $X_i$, respectively. Causal graphs are assumed to be \glspl{dag} \citep{pearl2009causality, spirtes2000causation}. 
Under the \emph{Markov} and \emph{faithfulness} assumptions, conditional 
independencies in the data correspond exactly to $d$-separation statements 
in $\mathcal{G}$.

\myparagraph{Markov equivalence.}
Multiple DAGs can encode the same conditional independencies and are thus 
indistinguishable from observational data. Such DAGs form a 
\emph{Markov equivalence class}, uniquely represented as a 
\gls{cpdag}, a mixed graph sharing the same skeleton and v-structures
(unshielded colliders $X_i \to X_k \leftarrow X_j$ with $X_i \not\sim X_j$\footnote{We write $A \sim B$ for adjacency (any edge between $A$ and $B$) and $A \not\sim B$ for non-adjacency.})
as all \glspl{dag} in the class \citep{andersson1997characterization,hauser2012characterization}. When
background knowledge constrains additional edge orientations, the equivalence
class is further refined into a \acrfull{mpdag} \citep{meek1995causal,wienobst2021extendability},
representing only those \glspl{dag} consistent with both data and prior constraints.

\myparagraph{Causal effects and adjustment sets.}
%
A central goal in causal inference is to estimate the effect of an intervention, $\doop{X\!=\!x}$~\citep{pearl1995causal,pearl2009causality}, which denotes externally setting $X$ to value $x$, on an outcome $Y$ using observational data. When the underlying causal graph is known, the \emph{generalized adjustment criterion}~\citep{perkovic2018complete,shpitser2010validity,vanderzander2019separators} provides a complete characterization of covariate sets that yield valid adjustment for estimating causal effects across \glspl{dag}, \glspl{cpdag}, \glspl{mag}, and \glspl{pag}. Among these valid adjustment sets, the \emph{optimal adjustment set} $\mathbf{O}^\star$ minimizes the asymptotic variance of covariate-adjusted estimators \citep{henckel2022graphical,witte2020efficient,rotnitzky2020efficient}, improving statistical efficiency. However, identifying $\mathbf{O}^\star$ requires a graph representation with sufficient edge orientation information.


%
%
%

\begin{figure}[t]
  \centering
  \resizebox{.9\columnwidth}{!}{%
  \begin{tikzpicture}[
      blackarrow/.style  = {line width=0.35pt, black,
                            -{Stealth[length=1.8mm,width=1.4mm,inset=0.4mm]}},
      blackline/.style   = {line width=0.35pt, black},
      blueline/.style    = {line width=1.6pt,  blue,
                            line cap=butt},
      bluearrow/.style   = {line width=1.6pt,  blue,
                            line cap=butt,
                            -{Stealth[length=2.4mm,width=2.4mm,inset=0.5mm]}},
      implies/.style     = {font=\normalsize},
      rulelabel/.style   = {font=\itshape},
      x=1cm, y=1cm,
    ]

    \def\gap{0.05}        
    \def\impliesgap{0.45} 
    \def\panelhalf{0.75}  
    \def\rulesep{1.5}     
    \def\rowsep{1.5}      
    \def\labeloffset{.8}  

    \coordinate (R1c) at (0, 0);
    \coordinate (R2c) at (\rulesep + 2*\panelhalf + 2*\impliesgap, 0);
    \coordinate (R3c) at (0, -\rowsep);
    \coordinate (R4c) at (\rulesep + 2*\panelhalf + 2*\impliesgap, -\rowsep);

    \def\Lx{-\panelhalf-\impliesgap}   
    \def\Rx{ \panelhalf+\impliesgap}   

    \node[rulelabel] at ($(R1c)+(0,\labeloffset)$) {R1};
    \node[implies]   at (R1c) {$\Rightarrow$};

    \begin{scope}[shift={($(R1c)+(\Lx,0)$)}]
      \draw[blackarrow] (-0.55, 0.95) -- (-0.55, 0.10);
      \draw[blueline]   (-0.55,-0.05) -- ( 0.55,-0.05);
    \end{scope}

    \begin{scope}[shift={($(R1c)+(\Rx,0)$)}]
      \draw[blackarrow] (-0.55, 0.95) -- (-0.55, 0.10);
      \draw[bluearrow]  (-0.55,-0.05) -- ( 0.55,-0.05);
    \end{scope}

    \node[rulelabel] at ($(R2c)+(0,\labeloffset)$) {R2};
    \node[implies]   at (R2c) {$\Rightarrow$};

    \begin{scope}[shift={($(R2c)+(\Lx,0)$)}]
      \coordinate (A) at (-0.55,-0.05);
      \coordinate (B) at (-0.55, 0.95);
      \coordinate (C) at ( 0.55,-0.05);
      \draw[blackarrow] (A) -- ($(B)+(0,-0.06)$);
      \draw[blackarrow] ($(B)+(0.07,-0.07)$) -- ($(C)+(-0.04,0.04)$);
      \draw[blueline]   ($(A)+(0.05,0)$) -- ($(C)+(-0.05,0)$);
    \end{scope}

    \begin{scope}[shift={($(R2c)+(\Rx,0)$)}]
      \coordinate (A) at (-0.55,-0.05);
      \coordinate (B) at (-0.55, 0.95);
      \coordinate (C) at ( 0.55,-0.05);
      \draw[blackarrow] (A) -- ($(B)+(0,-0.06)$);
      \draw[blackarrow] ($(B)+(0.07,-0.07)$) -- ($(C)+(-0.04,0.04)$);
      \draw[bluearrow]  ($(A)+(0.05,0)$) -- ($(C)+(-0.08,0)$);
    \end{scope}

    \node[rulelabel] at ($(R3c)+(0,\labeloffset)$) {R3};
    \node[implies]   at (R3c) {$\Rightarrow$};

    \def\sqh{0.55} 

    \begin{scope}[shift={($(R3c)+(\Lx,0)$)}]
      \coordinate (A) at (-\sqh,  \sqh);
      \coordinate (B) at ( \sqh,  \sqh);
      \coordinate (C) at (-\sqh, -\sqh);
      \coordinate (D) at ( \sqh, -\sqh);

      \draw[blackline]  ($(A)+(\gap,0)$)  -- ($(B)+(-\gap,0)$);
      \draw[blackline]  ($(A)+(0,-\gap)$) -- ($(C)+(0, \gap)$);
      \draw[blackarrow] ($(B)+(0,-\gap)$) -- ($(D)+(0, \gap)$);
      \draw[blackarrow] ($(C)+(\gap,0)$)  -- ($(D)+(-\gap,0)$);
      \draw[blueline]   ($(A)+( 0.07,-0.07)$) -- ($(D)+(-0.07, 0.07)$);
    \end{scope}

    \begin{scope}[shift={($(R3c)+(\Rx,0)$)}]
      \coordinate (A) at (-\sqh,  \sqh);
      \coordinate (B) at ( \sqh,  \sqh);
      \coordinate (C) at (-\sqh, -\sqh);
      \coordinate (D) at ( \sqh, -\sqh);
      \draw[blackline]  ($(A)+(\gap,0)$)  -- ($(B)+(-\gap,0)$);
      \draw[blackline]  ($(A)+(0,-\gap)$) -- ($(C)+(0, \gap)$);
      \draw[blackarrow] ($(B)+(0,-\gap)$) -- ($(D)+(0, \gap)$);
      \draw[blackarrow] ($(C)+(\gap,0)$)  -- ($(D)+(-\gap,0)$);
      \draw[bluearrow]  ($(A)+( 0.07,-0.07)$) -- ($(D)+(-0.10, 0.10)$);
    \end{scope}

    \node[rulelabel] at ($(R4c)+(0,\labeloffset)$) {R4};
    \node[implies]   at (R4c) {$\Rightarrow$};

    \begin{scope}[shift={($(R4c)+(\Lx,0)$)}]
      \coordinate (A) at (-\sqh,  \sqh);
      \coordinate (B) at ( \sqh,  \sqh);
      \coordinate (C) at ( \sqh, -\sqh);
      \coordinate (D) at (-\sqh, -\sqh);

      \draw[blackline]  ($(A)+(\gap,0)$)  -- ($(B)+(-\gap,0)$);
      \draw[blackarrow] ($(B)+(0,-\gap)$) -- ($(C)+(0, \gap)$);
      \draw[blackarrow] ($(C)+(-\gap,0)$) -- ($(D)+( \gap,0)$);
      \draw[blackline] ($(A)+( 0.07,-0.07)$) -- ($(C)+(-0.07, 0.07)$);
      \draw[blueline]   ($(A)+(0,-\gap)$) -- ($(D)+(0, \gap)$);
    \end{scope}

    \begin{scope}[shift={($(R4c)+(\Rx,0)$)}]
      \coordinate (A) at (-\sqh,  \sqh);
      \coordinate (B) at ( \sqh,  \sqh);
      \coordinate (C) at ( \sqh, -\sqh);
      \coordinate (D) at (-\sqh, -\sqh);
      \draw[blackline]  ($(A)+(\gap,0)$)  -- ($(B)+(-\gap,0)$);
      \draw[blackarrow] ($(B)+(0,-\gap)$) -- ($(C)+(0, \gap)$);
      \draw[blackarrow] ($(C)+(-\gap,0)$) -- ($(D)+( \gap,0)$);
      \draw[blackline] ($(A)+( 0.07,-0.07)$) -- ($(C)+(-0.07, 0.07)$);
      \draw[bluearrow]  ($(A)+(0,-\gap)$) -- ($(D)+(0, 0.10)$);
    \end{scope}

  \end{tikzpicture}%
  }
  \caption{Meek's four orientation rules \citep{meek1995causal}.
  Blue lines are undirected edges; blue arrows indicate the newly oriented edge implied by each rule. Nodes not adjacent in the precondition pattern are assumed non-adjacent.}
  \label{fig:meek-rules}
  \vspace{-1.5em}
\end{figure}
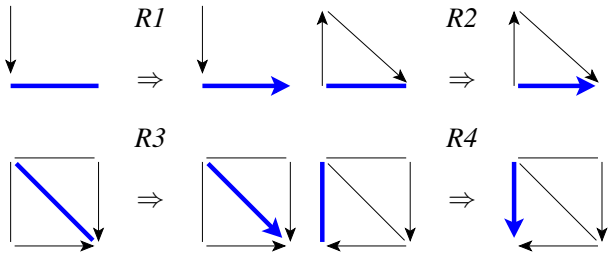
\myparagraph{Meek's orientation rules.}
After v-structures are identified, four deterministic rules~\citep{meek1995causal} are applied exhaustively in any order to propagate edge orientations to yield the unique CPDAG of the equivalence class (\Cref{fig:meek-rules}). Each rule orients an undirected edge 
to prevent either a new v-structure or a directed cycle:
\textbf{R1} orients $\beta\!-\!\gamma$ as $\beta\!\to\!\gamma$ when 
$\alpha\!\to\!\beta$ and $\alpha\not\sim\gamma$;
\textbf{R2} orients $\alpha\!-\!\gamma$ when a directed path $\alpha\!\to\!\beta\!\to\!\gamma$ would otherwise create a cycle; \textbf{R3} and \textbf{R4} handle configurations with two directed paths converging on a common node, orienting the remaining undirected edge to avoid spurious colliders. These same rules are applied when incorporating background knowledge: forced orientations from prior constraints are propagated through R1--R4 to derive all logically entailed directions, yielding an MPDAG.

\section{Related Work}\label{sec:related_work}
\myparagraph{Global and local causal structure learning.}
For a recent survey see \citet{heinzedeml2018causal}. Constraint-based methods such as PC and FCI~\cite{spirtes2000causation, zhang2008completeness}, score-based approaches such as GES~\cite{chickering2002optimal}, and hybrid score/constraint methods such as MMHC~\cite{tsamardinos2006mmhc} recover Markov equivalence classes from observational data but suffer from poor scalability due to the super-exponential size of the search space~\cite{robinson1977counting}. Local causal discovery addresses this by restricting inference to the neighborhood of a target variable. Sequential methods such as MB-by-MB~\cite{wang2014discovering} and the HITON/IAMB Markov-blanket family~\cite{aliferis2010local} iteratively expand the local structure, while parallel approaches such as DCILP~\cite{dong2025dcilp} learn and reconcile overlapping subgraphs. While these methods improve computational efficiency, they do not directly address identifiability or statistical optimality of causal effect estimation from partially oriented structures.

\myparagraph{Local optimal adjustment set discovery.} \citet{schubert2025local} introduce LOAD, which takes a step further by combining local discovery with the theory of optimal adjustment sets \cite{henckel2022graphical}. Building on MB-by-MB \cite{wang2014discovering}, LOAD identifies settings where the causal effect of $X$ on $Y$ is globally identifiable, using graphical criteria of amenability and, when so, recovers both the direction of causality and the associated optimal adjustment set from local structural information alone. LOAD thus bridges the gap between scalable local discovery and statistically efficient causal inference. However, as discussed in~\cref{sec:preliminaries}, its purely data-driven nature limits performance in data-scarce or structurally complex settings, and its guarantees are confined to settings of explicit ancestry.

\myparagraph{Background knowledge in causal discovery.}
The integration of prior knowledge into causal discovery has a rich theoretical foundation~\citep{heckerman1995learning}. \citet{meek1995causal} establishes the formal framework for causal inference under background knowledge, showing how required and forbidden edge orientations can be consistently incorporated into a CPDAG to yield an MPDAG, with orientation constraints propagated via the four Meek rules~(\Cref{fig:meek-rules})~\cite{meek1995causal}. \citet{perkovic2017interpreting} further develop this theory by providing tools for interpreting and manipulating~\glspl{cpdag} augmented with background knowledge, and notably underscores the importance of Meek's fourth rule, which is often omitted in implementations, for obtaining a maximally oriented representation. Following their ``b-'' prefix convention for background-knowledge analogues, we name our method \method. Beyond pairwise edge constraints, \citet{borboudakis2012incorporating} generalize the formalism to path-level (ancestral) constraints, including in MAGs. These works collectively establish that data-driven discovery alone has inherent identifiability limitations and that structured prior knowledge provides a principled mechanism for resolving otherwise undetermined edge orientations.

\myparagraph{Local discovery with background knowledge.} Most closely related to our work is the recent contribution of \citet{zheng2026local}, who extend MB-by-MB to incorporate background knowledge within a local causal discovery framework, grounding the approach in theoretical soundness and producing locally consistent~\glspl{mpdag} (their construction underlies the \textsc{LocalMPDAG} procedure of \Cref{alg:bload}). \citet{fang2020ida} independently propose a fully local IDA-style causal-effect estimator under background knowledge in MPDAGs, but stop short of identifying statistically optimal adjustment sets. Like its predecessors, neither method extends to the recovery of optimal adjustment sets. Our proposed framework precisely addresses this gap by combining knowledge-informed local structure refinement with the optimality criteria of \citet{henckel2022graphical}, enabling both the principled incorporation of prior knowledge and statistically efficient causal inference from local information.

As shown in~\cref{tab:related_work_comparison}, \method\ is the first framework to jointly support local discovery, background knowledge integration, \gls{mpdag} reasoning, and statistically optimal adjustment set identification.

\begin{table}[!htpb]
\centering
\caption{Summary of related work}\label{tab:related_work_comparison}
\resizebox{\linewidth}{!}{
\begin{tabular}{lccccccc}
  \toprule
  \multicolumn{1}{c}{\multirow{2}{*}{Method}} & Local & Background & \Gls{mpdag} & Effect & Optimal Adj. & Efficient & \multirow{2}{*}{Scalable} \\
  & Discovery & Knowledge & Reasoning & Identifiability & Set Discovery & Estimation & \\
  \midrule
    PC/FCI~\cite{spirtes2000causation,zhang2008completeness} & \xmark & Post-hoc & \cmark & Partial & \xmark & \xmark & \xmark \\
    GES~\cite{chickering2002optimal} & \xmark & \xmark & \xmark & Partial & \xmark & \xmark & \xmark \\
    MB-by-MB~\cite{wang2014discovering} & \cmark & \xmark & \xmark & \xmark & \xmark & \xmark & \cmark \\
    DCILP~\cite{dong2025dcilp} & \cmark & \xmark & \xmark & \xmark & \xmark & \xmark & \cmark \\
    LOAD~\cite{schubert2025local} & \cmark & \xmark & Partial & \cmark & \cmark & \cmark & \cmark \\
    \citet{zheng2026local} & \cmark & \cmark & \cmark & Partial & \xmark & \xmark & \cmark \\
    \citet{fang2020ida} & \cmark & \cmark & \cmark & Partial & \xmark & \xmark & \cmark \\
  \midrule
  \textbf{\method\ (ours)} & \cmark & \cmark & \cmark & \cmark & \cmark & \cmark & \cmark \\
\bottomrule
\end{tabular}
}
\end{table}

\section{Method}\label{sec:method}

\subsection{Theoretical Motivations}
\label{subsec:theory}

Three observations shape the design of \method.

\emph{(i) Optimal adjustment requires orientation.}
Optimal-adjustment-set theory~\citep{henckel2022graphical} requires a
sufficiently oriented graph: every proper causal path from treatment $X$ to
outcome $Y$ must begin with a directed edge out of $X$, the amenability
condition of \citet{perkovic2018complete}. Background knowledge supplies the
missing orientations precisely in the regimes---finite samples, dense local
neighborhoods---where data alone underdetermines them.

\emph{(ii) MPDAG manipulation needs the four-rule Meek closure.}
Manipulating a CPDAG under background constraints requires the four-rule
closure of \citet{perkovic2017interpreting} to obtain a maximally oriented
MPDAG. The three-rule closure propagates
v-structures and acyclicity but misses orientations implied jointly by two
directed paths converging on a common node, captured by R4---exactly the
configurations that background-knowledge edges create.

\emph{(iii) Local search must be re-driven by every orientation.}
Theorem~1 of \citet{zheng2026local} establishes that MB-by-MB-style local
search remains sound under the background knowledge provided the frontier of
nodes to expand is recomputed after every orientation step. Prior knowledge
must therefore drive the search trajectory itself, not patch its output: a
post-hoc orientation pass on a learned CPDAG cannot recover the nodes that
would have been reached through a background-induced orientation.

\Cref{lem:soundness,prop:identifiability} formalize the end-to-end
guarantees that follow from composing these three ingredients in
\Cref{alg:bload}.

\begin{lemma}[Soundness and Completeness of Local MPDAG Discovery]
\label{lem:soundness}
Let $\mathcal{G}$ be the true underlying DAG over a set of variables,
assuming causal sufficiency. Let $\mathcal{D}$ be observational data
generated from a distribution that is Markovian and faithful with respect
to $\mathcal{G}$, and let $\mathcal{B}$ be a set of valid background
knowledge. Denote by $\mathcal{G}^*$ the true MPDAG representing the
Markov equivalence class of $\mathcal{G}$ restricted by $\mathcal{B}$,
and let $X$ be the target variable. Assume all conditional independencies
are correctly identified.

If \method's local discovery phase is executed with inputs
$(X, \mathcal{D}, \mathcal{B})$ and produces output graph
$\mathcal{G}_{\mathrm{out}}$, then for $X$ and every node $Z$ connected
to $X$ by an undirected path in $\mathcal{G}^*$:
\begin{align}
    \Pa(Z,\, \mathcal{G}_{\mathrm{out}}) &= \Pa(Z,\, \mathcal{G}^*), \\
    \Ch(Z,\, \mathcal{G}_{\mathrm{out}}) &= \Ch(Z,\, \mathcal{G}^*), \\
    \Sib(Z,\, \mathcal{G}_{\mathrm{out}}) &= \Sib(Z,\, \mathcal{G}^*).
\end{align}
\end{lemma}

\begin{remark}
\label{rem:irreducible}
\Cref{lem:soundness} implies that any failure of \method\ to identify the
causal effect or the optimal adjustment set is attributable to
irreducible Markov equivalence---edges that remain undirected even in
$\mathcal{G}^*$---rather than to errors in local discovery. Background
knowledge strictly reduces this equivalence class, and therefore
monotonically improves identifiability prospects, as formalized next.
\end{remark}

\begin{proposition}[Background Knowledge Improves Identifiability]
\label{prop:identifiability}
Let $\mathcal{C}$ be the CPDAG of the Markov equivalence class of
$\mathcal{G}$, and let $\mathcal{M} = \mathcal{M}(\mathcal{C}, \mathcal{B})$
be the MPDAG obtained by incorporating valid background knowledge
$\mathcal{B} \neq \emptyset$ via Meek's rules R1--R4. Let $(X, Y)$ be
a causal query. Then:

\begin{enumerate}[label=(\roman*)]
    \item \textbf{Monotone refinement.} The equivalence class represented
    by $\mathcal{M}$ is a strict subset of that represented by $\mathcal{C}$:
    every DAG consistent with $\mathcal{M}$ is consistent with $\mathcal{C}$,
    but not vice versa.

    \item \textbf{Improved identifiability.} If the causal effect of $X$
    on $Y$ is identifiable in $\mathcal{C}$ (i.e., $\mathcal{C}$ is
    amenable w.r.t.\ $(X,Y)$), it remains identifiable in $\mathcal{M}$.
    Conversely, there exist configurations in which the effect is not
    identifiable in $\mathcal{C}$ but becomes identifiable in
    $\mathcal{M}$, strictly expanding the set of queries for which
    \method\ returns an answer.

    \item \textbf{Optimal adjustment set preservation.} If $\ \mathbf{O}^\star$
    is the optimal adjustment set identified from $\mathcal{C}$, then
    $\mathbf{O}^\star$ remains a valid adjustment set in $\mathcal{M}$.
    Furthermore, additional orientations in $\mathcal{M}$ may yield a
    strictly smaller or lower-variance optimal adjustment set than that
    obtainable from $\mathcal{C}$ alone.
\end{enumerate}
\end{proposition}

\subsection{Modelling Background Knowledge}
\label{subsec:sampling}

\paragraph{Working assumption: partial but perfect knowledge.}
We model background knowledge as a partial, fully correct snapshot of the
true causal graph: any constraint supplied to \method\ is consistent with the
data-generating DAG $\mathcal{G}$. Formally, given a true DAG
$\mathcal{G} = (\mathcal{V}, \mathcal{E})$, background knowledge is a set
\[
    \mathcal{B}^+ \;\subseteq\; \mathcal{E}
\]
of required directed edges, so that every orientation supplied to \method\
is a true directed edge of $\mathcal{G}$. This formalizes the standard
correctness assumption in the background-knowledge causal-discovery
literature~\citep{meek1995causal,perkovic2017interpreting}. As a deliberate
design choice, we restrict ourselves to the required edges in this work---the
most common form of curated prior knowledge in practice---and leave
forbidden edges to the data-driven skeleton step.

The motivation is empirical. In high-expertise application domains, curated
knowledge graphs and expert-elicited constraints typically encode a subset
of well-established causal relationships rather than a complete description
of the system. Our working assumption is faithful to this regime, where the
practical limitation is \emph{coverage} rather than \emph{correctness}.
This is precisely the setting in which our real-data evaluations
(\Cref{sec:empirical_evaluation}) operate.

We deliberately stress-test this assumption rather than rely on it:
\Cref{sec:discussion,app:noised} report a preliminary
analysis of \method\ under controlled corruption of $\mathcal{B}^+$,
characterizing the degradation regime as the knowledge becomes noisy.

\smallskip
\myparagraph{Two sampling strategies for $\mathcal{B}^+$.}
Given a query $(X, Y)$, where in $\mathcal{G}$ does a practitioner's prior
knowledge actually live? Two stylized scenarios capture the dominant
practical cases. Both are parameterized by a fraction $\rho \in (0, 1)$
controlling the proportion of edges from the eligible pool revealed to
\method\ as background knowledge.

\smallskip
\noindent\emph{Local sampling
strategy.}\footnote{Not to be confused with \emph{local causal discovery},
which describes \method\ itself.}
$\mathcal{B}^+$ is sampled uniformly from edges of $\mathcal{E}$
\emph{incident to either $X$ or $Y$}. This simulates a practitioner who possesses incomplete but highly accurate
domain knowledge about the immediate neighborhoods of $X$ and $Y$.

\smallskip
\noindent\emph{Global sampling strategy.}
$\mathcal{B}^+$ is sampled uniformly from the full edge set $\mathcal{E}$
without conditioning on $X$ or $Y$. This simulates the scenario in which a
practitioner has access to a general domain knowledge graph that may or
may not happen to touch the variables of the specific query.

Together the two strategies let us disentangle \emph{how much} prior
knowledge is available from \emph{where} it lies, providing a controlled
instrument for the experiments of \Cref{sec:empirical_evaluation}.

\subsection{The \method\ Algorithm}
\label{subsec:bload}

\Cref{alg:bload} presents \method. The outer loop is that of
LOAD~\citep{schubert2025local}: invoke local discovery around each target
via \textsc{LocalMPDAG}, check that $X$ is a proper ancestor of $Y$ and that
the resulting graph is amenable w.r.t.\ $(X,Y)$~\citep{perkovic2018complete},
then extract the optimal adjustment set via the criterion of
\citet{henckel2022graphical}. Background knowledge enters the inner
procedure \textsc{LocalMPDAG} at (a)--(d).

\begin{algorithm}[t]
\caption{\textbf{\method}: Background knowledge-augmented Local Optimal Adjustment set Discovery}
\label{alg:bload}
\label{alg:mbbymbmpdag}
\begin{algorithmic}[1]
\Require Targets $X, Y$; data $\mathcal{D}$; CI oracle $\mathcal{I}$; background knowledge $\mathcal{B}^+$
\Ensure Optimal adjustment set $\mathbf{O}^\star$, or \textsc{NotIdentifiable}
\Statex \textit{// Knowledge-aware local discovery around each target}
\State $G_X \leftarrow \Call{LocalMPDAG}{X, \mathcal{D}, \mathcal{B}^+, \mathcal{I}}$
\State $G_Y \leftarrow \Call{LocalMPDAG}{Y, \mathcal{D}, \mathcal{B}^+, \mathcal{I}}$
\State $\mathcal{M}_{\mathrm{loc}} \leftarrow G_X \cup G_Y$
\Statex \textit{// Identifiability and optimal-adjustment-set extraction}
\If{$X$ is not a proper ancestor of $Y$ in $\mathcal{M}_{\mathrm{loc}}$}
    \State \Return \textsc{NotIdentifiable}
\EndIf
\If{$\mathcal{M}_{\mathrm{loc}}$ is not amenable w.r.t.\ $(X, Y)$~\citep{perkovic2018complete}}
    \State \Return \textsc{NotIdentifiable}
\EndIf
\State $\mathbf{O}^\star \leftarrow \Call{OptimalAdjSet}{X, Y, \mathcal{M}_{\mathrm{loc}}}$ \Comment{\citep{vanderzander2019separators,witte2020efficient}}
\State \Return $\mathbf{O}^\star$
\vspace{0.3em}\hrule\vspace{0.3em}
\Procedure{LocalMPDAG}{$Z, \mathcal{D}, \mathcal{B}^+, \mathcal{I}$}
    \Statex \quad\textit{// (a) Initialize with background knowledge, not the empty graph}
    \State $G \leftarrow (\mathcal{V}, \mathcal{B}^+)$;\quad
           $\textsc{Done} \leftarrow \emptyset$;\quad
           $\textsc{Wait} \leftarrow \{Z\}$
    \While{$\textsc{Wait} \neq \emptyset$}
        \State Pop $U$ from $\textsc{Wait}$; add $U$ to $\textsc{Done}$
        \State Find $\mathrm{MB}(U)$; learn marginal graph $L_U$ over $\mathrm{MB}(U) \cup \{U\}$
        \Statex \quad\quad\textit{// (b) BK-protected update: re-impose $\mathcal{B}^+$ after every local edit}
        \State Add edges and v-structures of $L_U$ involving $U$ to $G$;\;\textbf{then} re-impose $\mathcal{B}^+$
        \Statex \quad\quad\textit{// (c) Four-rule Meek closure for MPDAGs (R4 essential)}
        \State Exhaustively apply Meek's rules R1--R4 to $G$ (\Cref{fig:meek-rules})
        \Statex \quad\quad\textit{// (d) BK-aware frontier expansion}
        \State $\begin{aligned}[t]
        \textsc{Wait} \leftarrow{}&
        \{\,U' \in \mathcal{V} \mid
        U' \text{ connected to } Z \\
        &\text{by an undirected path in } G\,\} \setminus \textsc{Done}
        \end{aligned}$
    \EndWhile
    \State \Return $G$
\EndProcedure
\end{algorithmic}
\end{algorithm}

\emph{(a) Initialization.}
Seeding the local graph with $\mathcal{B}^+$ (rather than the empty graph)
makes prior orientations available to the first v-structure check and Meek
pass, so they enter the search trajectory rather than patching its output.

\emph{(b) Protected updates.}
After each Markov-blanket discovery and orientation step, $\mathcal{B}^+$ is
re-imposed on the working graph. This serves two purposes. First, it prevents
finite-sample errors of the PC-style skeleton substep from deleting required
edges and thereby invalidating subsequent Meek propagation. Second, it
resolves any conflict between a data-driven orientation and a constraint in
$\mathcal{B}^+$ in favor of $\mathcal{B}^+$: the re-imposition overwrites the
conflicting data-driven orientation, so the working graph after step~(b) is
consistent with $\mathcal{B}^+$ by construction. Because
$\mathcal{B}^+ \subseteq \mathcal{E}$ for the true acyclic DAG $\mathcal{G}$,
the subset of edges fixed by $\mathcal{B}^+$ is itself acyclic, and the
conflict-resolution rule guarantees that no directed cycle can be introduced
by the imposition of background knowledge. Edges outside $\mathcal{B}^+$ are
oriented by data-driven v-structure detection and Meek propagation, both of
which are sound under the standard assumptions (Theorem~1 of
\citet{zheng2026local}) and preserve acyclicity by construction of the rules
themselves \citep{meek1995causal}.

\emph{(c) Four-rule Meek closure.}
Propagation uses the four-rule closure of \citet{perkovic2017interpreting};
R4 is what lets data-driven and BK-driven orientations compound into a
maximally oriented MPDAG rather than merely coexist.

\emph{(d) Frontier expansion.}
Recomputing the wait-list after every propagation step lets BK-induced
orientations shrink the frontier (by directing previously undirected edges)
and enlarge it (by revealing newly relevant nodes)---the soundness
condition of Theorem~1 of \citet{zheng2026local}.

\section{Empirical Evaluation}\label{sec:empirical_evaluation}

We evaluate \method\ against three baselines that, together, span the $2{\times}2$ design of \{global, local\} structure learning $\times$ \{without, with\} background knowledge:
\begin{itemize}[leftmargin=*,nosep]
    \item \textbf{PC}~\citep{spirtes2000causation}: global, data-driven.
    \item \textbf{b-PC}: naive augmentation of PC with the four-rule Meek propagation
          of $\mathcal{B}$; the natural ``global learner $+$ BK'' baseline.
    \item \textbf{LOAD}~\citep{schubert2025local}: local, data-driven.
    \item \textbf{\method}: local, knowledge-informed (this work).
\end{itemize}
This design isolates whether the empirical gains attributable to \method\ stem from local discovery, from background knowledge, or from their synergy. Where applicable, each knowledge-augmented method is further evaluated under both \emph{local} and \emph{global} BK sampling (see \cref{subsec:sampling}). All experiments are run over $1000$ random seeds to obtain mean estimates with tight confidence intervals.%

\subsection{Experimental Setup}
\label{subsec:setup}

\myparagraph{Background knowledge sampling.}
We instantiate the local and global sampling strategies of
\cref{subsec:sampling} at every fraction
$\rho \in \{0.1, 0.2, \ldots, 0.9\}$, running each
$(\rho,\,\text{strategy})$ configuration independently. Sampled edge
orientations are assumed correct in every configuration except in the
preliminary noise-robustness analysis of \Cref{app:noised}, where we corrupt a fraction of them.

\myparagraph{Evaluation metrics.}
We report two complementary metrics. The \emph{F1 score on optimal
adjustment set identification} measures how accurately each method
recovers the ground-truth optimal adjustment set $\mathbf{O}^\star$,
balancing precision and recall over its constituent variables. As a
downstream metric we adopt the \emph{Intervention Distance} of
\citet{schubert2025snap}, defined as the distance between the
ground-truth causal effect $\Theta_{T,T'}$ from a target $T$ to $T'$
and the predicted causal effect $\widehat{\Theta}_{T,T'}$ for the same
pair obtained by adjusting on the returned set. Unlike F1, which
scores the \emph{composition} of the adjustment set, the intervention
distance scores the quantity the practitioner actually consumes---the
estimated causal effect---and is the metric of primary
interest in the downstream use of causal discovery.%

\myparagraph{Statistical power metric.}
Following \citet{kummerfeld2024power}, we summarize experimental conditions 
along a unified axis of \emph{statistical power per structural parameter}, 
defined as:
\begin{equation}
    \Pi = \frac{n}{\lvert \mathcal{V} \rvert \cdot d_{\mathrm{exp}} / 2},
    \label{eq:statpower}
\end{equation}
where $n$ is the number of samples, $\lvert \mathcal{V} \rvert$ the number 
of observed variables, and $d_{\mathrm{exp}}$ the expected node degree. 
The denominator approximates the total number of edges in the graph\footnote{The approximation $|\mathcal{V}| \cdot d_{\mathrm{exp}}/2$ is exact in expectation for Erd\H{o}s--R\'enyi graphs $G(n, d_{\mathrm{exp}}/n)$ and a tight estimate per realized graph.}, so
$\Pi$ represents the average number of samples available per edge to be
learned. Low $\Pi$ corresponds to data-scarce, high-complexity regimes 
where purely data-driven methods are expected to struggle; high $\Pi$ 
corresponds to favorable regimes where LOAD and \method\ should converge. 
This metric allows results across heterogeneous graph configurations to 
be collapsed onto a single interpretable axis.

\subsection{Synthetic Experiments}
\label{subsec:synthetic}

\myparagraph{Data generation.}
We generate random DAGs using the Erdős–Rényi model, providing fine-grained 
control over three axes of difficulty: the number of nodes 
$\lvert \mathcal{V} \rvert \in \{10, 15, 20, 25, 50, 100\}$, the expected 
degree $d_{\mathrm{exp}} \in \{2, 3, 4, 5, 6\}$, and the number of 
observational samples $n \in \{100, 200, 500, 1000, 2000, 5000, 10000\}$. 
Linear Gaussian structural equations are used throughout, with edge weights 
sampled uniformly from $[-1, -0.25] \cup [0.25, 1]$ to avoid near-zero 
effects. For each configuration, causal queries $(X, Y)$ are drawn uniformly 
from ancestral pairs in $\mathcal{G}$.

\myparagraph{F1 versus statistical power.}
\Cref{fig:load_comparison} resolves F1 along the statistical-power
axis $\Pi$ (\Cref{eq:statpower}) for representative BK fractions
$\rho$, aggregating across graph sizes, edge densities, and sample
sizes. Both \method\ variants dominate every panel and are nearly
indistinguishable, jointly forming the top curve at all $\rho$.
LOAD and the two b-PC variants occupy the middle band---and
notably, b-PC-local edges out b-PC-global on F1, the inverse of
the pattern under \method, suggesting that globally distributed
constraints help F1 most when local search is doing its job. PC is
the lowest curve throughout. The relative gain of \method\ is
largest at low $\Pi$, where data-driven CPDAGs are weakly oriented,
and shrinks at high $\Pi$ as PC and LOAD recover well-oriented
CPDAGs from data alone, consistent with \Cref{prop:identifiability}.
Ultimately, this demonstrates that injecting background knowledge via 
\method\ provides the most substantial benefit in challenging, low-signal 
regimes where purely observational data is insufficient for robust causal discovery.

\begin{figure}[t]
    \centering
    \includegraphics[width=\columnwidth]{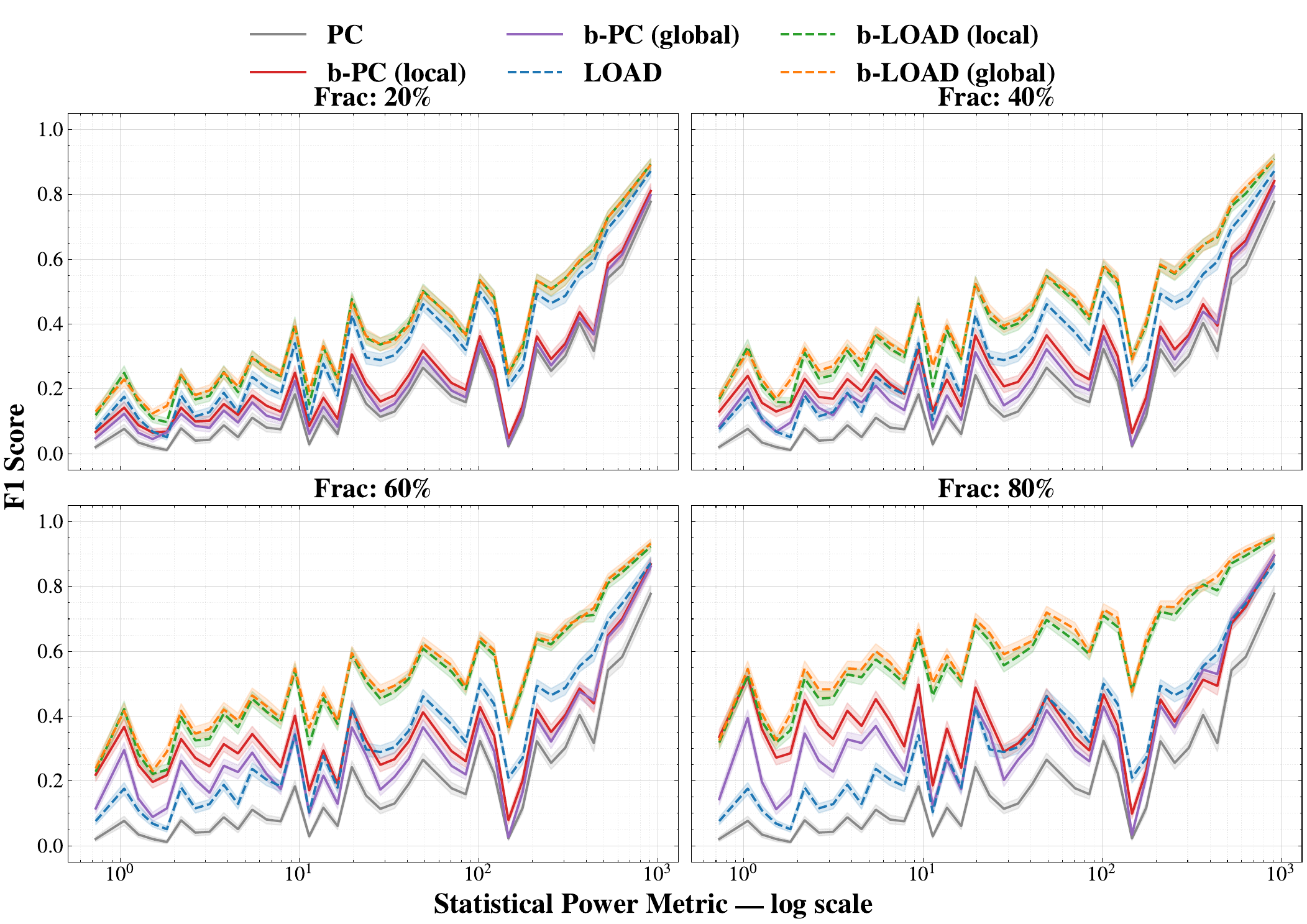}
    \caption{F1 vs.\ statistical power $\Pi$ (\Cref{eq:statpower})
    for a $2{\times}2$ grid of BK fractions $\rho \in \{20\%, 40\%, 60\%, 80\%\}$
    and method variants (higher is better). \method\ (global, orange dashed) and
    \method\ (local, green dashed) overlap at the top of every panel.
    Shaded regions: 95\% CI over 1000 seeds.}
    \label{fig:load_comparison}
    \vspace{-10pt}
\end{figure}%

\subsubsection{Data Efficiency}
\label{subsubsec:data_efficiency}

\Cref{fig:synth-A} plots Intervention Distance against sample size
$n$ for a graph size of $|\mathcal{V}|=10$ nodes and expected node degree $d_{\mathrm{exp}}=2$. The central takeaway is that
incorporating background knowledge acts as a powerful sample-size
substitute, particularly in the low-data regime where conditional
independence tests are traditionally unreliable. While the \method\ variants and the baseline LOAD algorithm
largely converge to similar performance levels at large sample sizes, a substantial performance gap exists when data is scarce.
Focusing on the early data regimes ($n \leq 10^3$), the \method\ variants
consistently achieve the lowest intervention distances. The local variant,
\method\ (local), performs exceptionally well here, visibly outperforming
both its global counterpart and standard LOAD. Notably, \method\ (local)
maintains a remarkably depressed intervention distance across these smaller 
sample sizes, demonstrating that the synergy of a local-discovery backbone 
and background knowledge dramatically accelerates accurate causal discovery 
well before large datasets become available.
\subsubsection{Sensitivity to Background Knowledge}
\label{subsubsec:bk_sensitivity}

\Cref{fig:synth-B} plots Intervention Distance against the BK
fraction $\rho$ across four sample sizes for a graph size of expected node degree $d_{\mathrm{exp}}=2$. Expectedly, the unaugmented 
PC and LOAD baselines remain flat. While all knowledge-augmented methods 
improve as $\rho$ increases, the \method\ variants demonstrate a distinct 
advantage at low background knowledge exposures (e.g., $\rho \leq 0.4$). 
In this scarce-knowledge regime, \method\ achieves a substantial 
reduction in error, significantly outpacing both b-PC variants. This stems from \method's structural design: it integrates $\mathcal{B}$ 
during the local learning process, re-propagating Meek's rules at every 
Markov-blanket update so that even sparse constraints cascade widely. 
Conversely, b-PC variants inject constraints post-hoc into an already-fixed 
CPDAG, yielding a much smaller propagation footprint. Thus, \method\ 
extracts a markedly higher marginal return from limited prior knowledge.

\subsubsection{Scale-Preserving Advantage}
\label{subsubsec:scalability}

\Cref{fig:synth-C} plots Intervention Distance against graph size $|\mathcal{V}|$
for four sample sizes, fixing the background knowledge fraction at 
$\rho=0.3$ and expected degree at $d_{\mathrm{exp}}=2$. While error naturally 
rises with $|\mathcal{V}|$ across all methods due to the proliferation of spurious
adjustment sets and unresolved orientations, the \method\ variants 
maintain a distinct scalability advantage. Crucially, the performance gap between \method\ and the baselines does 
not collapse at scale; it actively widens as $|\mathcal{V}|$ increases, particularly
in lower-data regimes ($n \in \{100, 1000\}$). The two \method\ variants 
perform almost identically across all graph sizes, consistently outperforming 
both b-PC variants. This widening gap highlights a structural advantage: 
\method\ uses targeted prior knowledge to prune false candidates before 
they multiply with graph size, bypassing the quadratic costs of full-skeleton 
revisions. This confirms that \method\ is 
simultaneously sample-efficient, knowledge-efficient, and scale-preserving. \\

\begin{figure*}[t]
    \centering
    \begin{subfigure}[t]{0.78\linewidth}
        \centering
        \includegraphics[width=\linewidth]{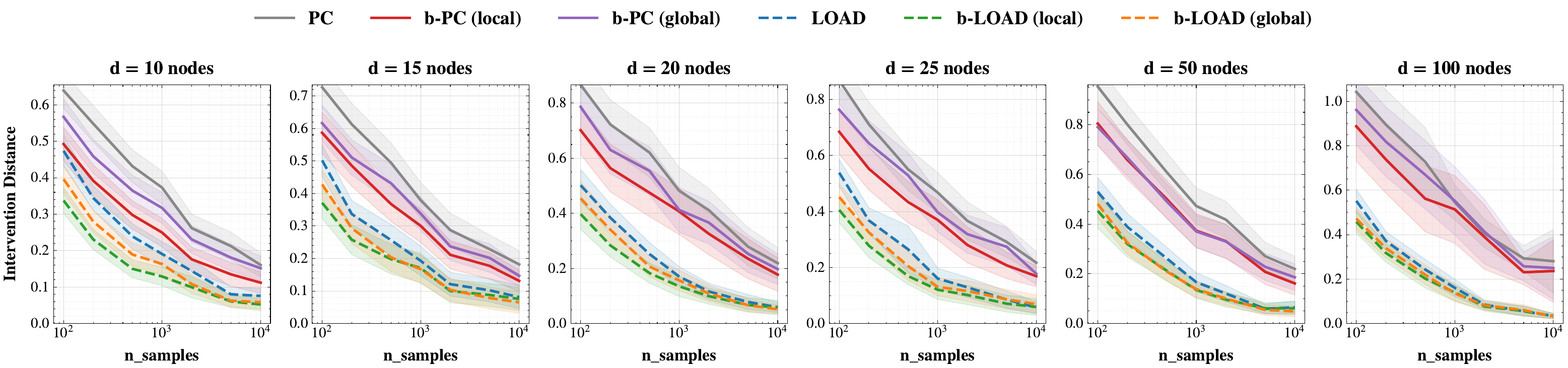}
        \caption{\textbf{Data efficiency.} Intervention Distance vs.\
        sample size $n$ (log scale) for $|\mathcal{V}| \in \{10, 15, 20, 25, 50, 100\}$
        at $\rho{=}0.3$ BK.}
        \label{fig:synth-A}
    \end{subfigure}

    \vspace{0.4em}

    \begin{subfigure}[t]{0.78\linewidth}
        \centering
        \includegraphics[width=\linewidth]{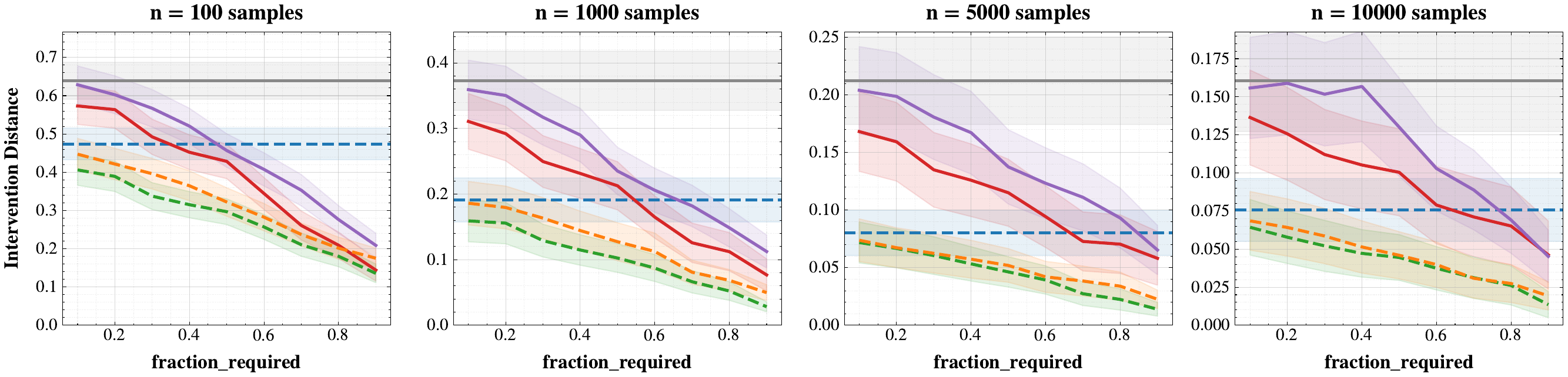}
        \caption{\textbf{BK sensitivity.} Intervention Distance vs.\
        BK fraction $\rho$ for $n \in \{100, 1000, 5000, 10000\}$ at
        $|\mathcal{V}|{=}10$. PC and LOAD are flat (BK-invariant); all four
        knowledge-augmented methods slope down, with the \method\
        slopes visibly steeper than either b-PC variant.}
        \label{fig:synth-B}
    \end{subfigure}

    \vspace{0.4em}

    \begin{subfigure}[t]{0.78\linewidth}
        \centering
        \includegraphics[width=\linewidth]{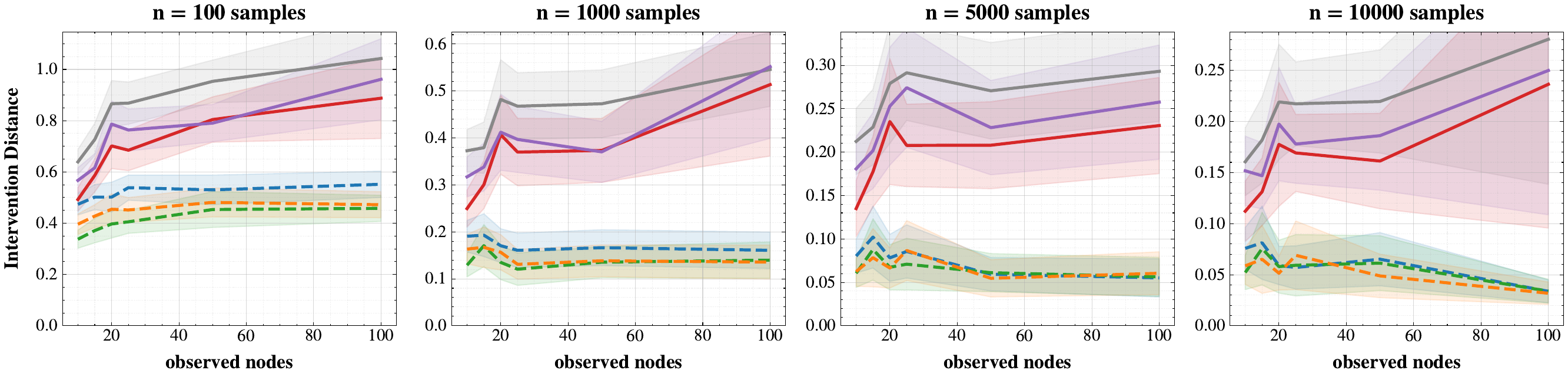}
        \caption{\textbf{Scale-preservation.} Intervention Distance vs.\
        observed nodes $|\mathcal{V}|$ for $n \in \{100, 1000, 5000, 10000\}$
        at $\rho{=}0.3$ BK.}
        \label{fig:synth-C}
    \end{subfigure}

    \caption{Intervention Distance (lower is better) on synthetic ER2 graphs across three
    orthogonal views: data efficiency~(\subref{fig:synth-A}),
    background-knowledge sensitivity~(\subref{fig:synth-B}), and
    scale-preservation~(\subref{fig:synth-C}). Both \method\ variants---local
    (green dashed) and global (orange dashed) BK sampling---attain the
    lowest error across all panels, with global slightly below local
    at small-to-moderate $n$ and the two near-overlapping at large $n$.}
    \label{fig:synth-combined}
\end{figure*}%

\myparagraph{Prescriptive summary by BK exposure and $\Pi$ regime.}
\Cref{tab:regime_results} evaluates the synthetic sweep through a background knowledge (BK) exposure lens, simulating different regimes of expert knowledge coverage. A clear trade-off emerges from the data. The \method\ variants dominate structure recovery, achieving the highest F1 scores across almost all exposure levels and $\Pi$ regimes. However, the optimal method for minimizing Intervention Distance shifts depending on knowledge availability. In the scarce-knowledge regime (10\%--30\%)---arguably the most realistic scenario for practitioners---\method-\emph{local} consistently secures the best Intervention Distance. Conversely, at Medium and High BK exposures, b-PC (global) surprisingly eclipses the other methods in Intervention Distance for low and medium $\Pi$ settings, even while \method\ maintains its F1 supremacy. We expand upon the practical implications of this divergence between structural accuracy and interventional performance in the discussion in \Cref{sec:discussion}.

\begin{table}[t]
\centering
\caption{Synthetic-data performance segmented by sample-complexity
regime: Low ($\Pi < 10$), Medium ($10 \leq \Pi < 100$), and High
($\Pi \geq 100$). Best entry in each column per BK exposure level is bolded.}
\label{tab:regime_results}
\resizebox{\columnwidth}{!}{%
\begin{tabular}{llcccccc}
\toprule
& & \multicolumn{2}{c}{\textbf{Low ($\Pi < 10$)}} & \multicolumn{2}{c}{\textbf{Medium ($10 \le \Pi < 100$)}} & \multicolumn{2}{c}{\textbf{High ($\Pi \ge 100$)}} \\
\cmidrule(lr){3-4} \cmidrule(lr){5-6} \cmidrule(lr){7-8}
\textbf{BK Exposure} & \textbf{Method} & \textbf{F1 $\uparrow$} & \textbf{Int. Dist. $\downarrow$} & \textbf{F1 $\uparrow$} & \textbf{Int. Dist. $\downarrow$} & \textbf{F1 $\uparrow$} & \textbf{Int. Dist. $\downarrow$} \\
\midrule
None & PC                 & $0.07$ & $2.90$ & $0.17$ & $2.15$ & $0.31$ & $1.28$ \\
     & LOAD               & $\mathbf{0.16}$ & $\mathbf{2.51}$ & $\mathbf{0.34}$ & $\mathbf{1.79}$ & $\mathbf{0.49}$ & $\mathbf{1.10}$ \\
\midrule
Low (10\% to 30\%)
& b-PC (local)          & $0.13$ & $2.73$ & $0.22$ & $1.85$ & $0.35$ & $1.07$ \\
& b-PC (global)         & $0.11$ & $2.63$ & $0.19$ & $1.78$ & $0.33$ & $1.05$ \\
& \method\ (local)      & $0.22$ & $\mathbf{2.42}$ & $\mathbf{0.39}$ & $\mathbf{1.70}$ & $\mathbf{0.54}$ & $\mathbf{0.99}$ \\
& \method\ (global)     & $\mathbf{0.23}$ & $2.52$ & $0.38$ & $1.75$ & $\mathbf{0.54}$ & $1.02$ \\
\midrule
Medium (40\% to 60\%)
& b-PC (local)          & $0.24$ & $2.42$ & $0.29$ & $1.58$ & $0.40$ & $0.91$ \\
& b-PC (global)         & $0.18$ & $\mathbf{2.07}$ & $0.24$ & $\mathbf{1.37}$ & $0.37$ & $\mathbf{0.81}$ \\
& \method\ (local)      & $0.32$ & $2.27$ & $0.47$ & $1.57$ & $0.61$ & $0.86$ \\
& \method\ (global)     & $\mathbf{0.34}$ & $2.40$ & $\mathbf{0.48}$ & $1.63$ & $\mathbf{0.62}$ & $0.92$ \\
\midrule
High (70\% to 90\%)
& b-PC (local)          & $0.39$ & $2.10$ & $0.37$ & $1.32$ & $0.45$ & $0.79$ \\
& b-PC (global)         & $0.28$ & $\mathbf{1.72}$ & $0.32$ & $\mathbf{1.10}$ & $0.44$ & $0.65$ \\
& \method\ (local)      & $0.50$ & $1.89$ & $0.62$ & $1.22$ & $0.73$ & $\mathbf{0.60}$ \\
& \method\ (global)     & $\mathbf{0.52}$ & $2.18$ & $\mathbf{0.64}$ & $1.37$ & $\mathbf{0.75}$ & $0.65$ \\
\bottomrule
\end{tabular}%
}
\end{table}

\subsection{Real-World Evaluation}
\label{subsec:realworld}

We evaluate \method\ on two canonical real-world benchmarks: the Sachs
protein-signaling dataset~\citep{sachs2005causal} and the five
multifactorial subnetworks of the DREAM4 \emph{In~Silico}
challenge~\citep{marbach2010dream4}. Together, they span distinct segments 
of the statistical-power axis $\Pi$ on which our synthetic results
have been organized (\Cref{tab:realworld-dims}): Sachs falls squarely into 
the medium-$\Pi$ regime ($\Pi \approx 51.7$), while DREAM4 sits at the extreme 
low end of the axis ($\Pi \approx 1$).

\begin{table}[t]
  \centering
  \caption{Dimensions and statistical-power regime of the two real-world
  benchmarks. $\Pi$ is computed as in \cref{eq:statpower}.}
  \label{tab:realworld-dims}
  
  \resizebox{\linewidth}{!}{
  \begin{threeparttable} 
    \begin{tabular}{lcccccc}
      \toprule
      Benchmark & $n$ & $|\mathcal{V}|$ & $d_{\mathrm{exp}}$ & $\Pi$ & Regime & Identifiable forward pairs\tnote{a} \\
      \midrule
      Sachs & $853$ & $11$ & $\approx 3$ & $\approx 51.7$ & Med & $46$ \\
      DREAM4 (per subnet) & $100$ & $100$ & $\approx 2$ & $\approx 1$ & Low & $329,\,517,\,681,\,1114,\,763$ \\
      \bottomrule
    \end{tabular}
    
    \begin{tablenotes}
      \item[a] These represent all $(X,Y)$ ancestral pairs (where $Y$ is a direct or indirect descendant of $X$) extracted from the true underlying DAG.
    \end{tablenotes}
  \end{threeparttable}%
  } 
\end{table}

\myparagraph{Curated knowledge graphs as a source of $\mathcal{B}^+$.}
In molecular-biology applications, curated knowledge graphs such as
Reactome~\citep{milacic2024reactome} catalog well-established biological
pathways. Each such pathway corresponds to a known directed edge in the
gold-standard regulatory or signaling network, so providing these
pathways as required edges in $\mathcal{B}^+$ is equivalent to drawing a
sample from the true edge set $\mathcal{E}$---the operationalization
studied in our synthetic experiments. The real-data sections below thus
instantiate the $\mathcal{B}^+ \subseteq \mathcal{E}$ modelling choice of
\cref{subsec:sampling}.

\myparagraph{Oracle ATE for Intervention Distance on real data.}
Real biological datasets provide observations and the gold-standard
DAG topology but not the underlying structural-equation weights, so
evaluating Intervention Distance against synthetic weights would
disconnect the ground truth from the data. We therefore construct an
\emph{oracle} linear~\gls{sem}: for every node $V$ in the gold-standard DAG we fit an ordinary least-squares regression of $V$ on its true parents,
and use the resulting coefficients as the oracle structural weights.
Intervention Distance is then the absolute difference between a
method's adjusted-ATE estimate and the oracle ATE obtained under the
same OLS estimator on the same data. Rather than serving as a redundant 
proxy for structural recovery, this metric directly quantifies the 
practical error in downstream causal effect estimation, operating under 
a first-order linear approximation of the underlying biological mechanisms.

The aggregate results on both benchmarks
(\Cref{tab:combined_results}, averaged over all $\rho$) reveal three
patterns. As a premise, to illustrate real-world utility against gold-standard models, 
F1 here measures the correct identification of adjustment sets on the \emph{True DAG} 
(rather than the CPDAG), and ``Queries identified'' represents the raw count of these 
relationships that successfully yield a valid adjustment set for accurate effect estimation. 
(i) Every BK-augmented method strictly improves over its purely data-driven counterpart 
across all metrics, mirroring the synthetic prescriptions of \Cref{tab:regime_results}. 
(ii) The optimal sampling strategy diverges based on the dataset. On the low-$\Pi$ DREAM4 
dataset, \method\ with \emph{global} sampling is exceptionally strong, identifying the most 
forward pairs ($340/3404$) while tying for the lowest Intervention Distance. However, on 
the medium-$\Pi$ Sachs dataset, b-PC (global) identifies the absolute most pairs ($11/46$), 
and \method-\emph{local} actually outperforms \method-global in both queries identified 
($8/46$ vs. $6/46$) and F1 score ($0.17$ vs. $0.13$). (iii) Despite these variations in 
graph-recovery accuracy, the practical choice for downstream causal-effect estimation 
remains \emph{local} sampling (see \Cref{sec:discussion}).

\begin{table}[t]
\centering
\caption{Aggregate results for the DREAM4 and Sachs datasets, averaged across BK fractions $\rho \in \{0.1, \ldots, 0.9\}$. Forward pairs are reported as ``identified / total identifiable''.}
\label{tab:combined_results}
\resizebox{\columnwidth}{!}{%
\begin{tabular}{lccc}
\toprule
\textbf{Method} & \textbf{Queries identified $\uparrow$} & \textbf{F1 (adj. set) $\uparrow$} & \textbf{Intervention Dist. $\downarrow$} \\
\midrule
\multicolumn{4}{c}{\textbf{DREAM4 (5 Subnetworks)}} \\
\midrule
LOAD                       & $26 / 3404$           & $0.01$          & $0.07$ \\
\method\ (global sampling) & $\mathbf{340} / 3404$ & $0.11$          & $\mathbf{0.04}$ \\
\method\ (local sampling)  & $\mathbf{338} / 3404$ & $0.11$          & $\mathbf{0.04}$ \\
\midrule
PC                         & $7 / 3404$            & $0.00$          & $0.07$ \\
b-PC (global sampling)     & $82 / 3404$           & $0.03$          & $0.06$ \\
b-PC (local sampling)      & $317 / 3404$          & $\mathbf{0.13}$ & $0.05$ \\
\midrule
\multicolumn{4}{c}{\textbf{Sachs}} \\
\midrule
LOAD                       & $0 / 46$              & $0.00$          & $0.14$ \\
\method\ (global sampling) & $6 / 46$              & $0.13$          & $\mathbf{0.07}$ \\
\method\ (local sampling)  & $8 / 46$              & $0.17$          & $\mathbf{0.07}$ \\
\midrule
PC                         & $0 / 46$              & $0.00$          & $0.10$ \\
b-PC (global sampling)     & $\mathbf{11} / 46$    & $\mathbf{0.24}$ & $\mathbf{0.07}$ \\
b-PC (local sampling)      & $7 / 46$              & $0.15$          & $\mathbf{0.07}$ \\
\bottomrule
\end{tabular}%
}
\end{table}

\subsubsection{Sachs: Results}
\label{subsec:sachs}

In the medium-$\Pi$ regime of Sachs, intricate signaling cascades make
the causal structure difficult to recover from observational data
alone; the test is whether \method\ extracts an additional,
complementary benefit from prior knowledge on top of any data-driven
baseline.

We compare all methods at every
$\rho \in \{0.1, \dots, 0.9\}$ under both global and local BK
sampling (\Cref{fig:sachs-results}). The results reveal a nuanced performance landscape where the optimal method shifts depending on how knowledge is sampled. On the F1 metric, while b-PC exhibits a sharp, late spike under global sampling at high $\rho$, \method\ under \emph{local} sampling demonstrates a highly robust and steady climb. It consistently outperforms b-PC locally across almost all fractions, ultimately reaching an impressive F1 of nearly $0.40$ at $\rho{=}0.9$. 

This strength of \method-\emph{local} is even more pronounced on the downstream Intervention Distance metric. While global sampling introduces some volatility for both b-methods, \emph{local} sampling yields exceptionally clean, monotonic improvements. Under this local strategy, \method's Intervention Distance decreases predictably as $\rho$ grows, edging out b-PC to achieve a near-zero error at $\rho{=}0.9$. Locally targeted prior knowledge combined with \method's learning loop therefore provides the most reliable and highest-performing causal-effect estimation on Sachs---an empirical nuance we unpack further in the discussion.

\begin{figure}[t]
    \centering
    \includegraphics[width=\columnwidth]{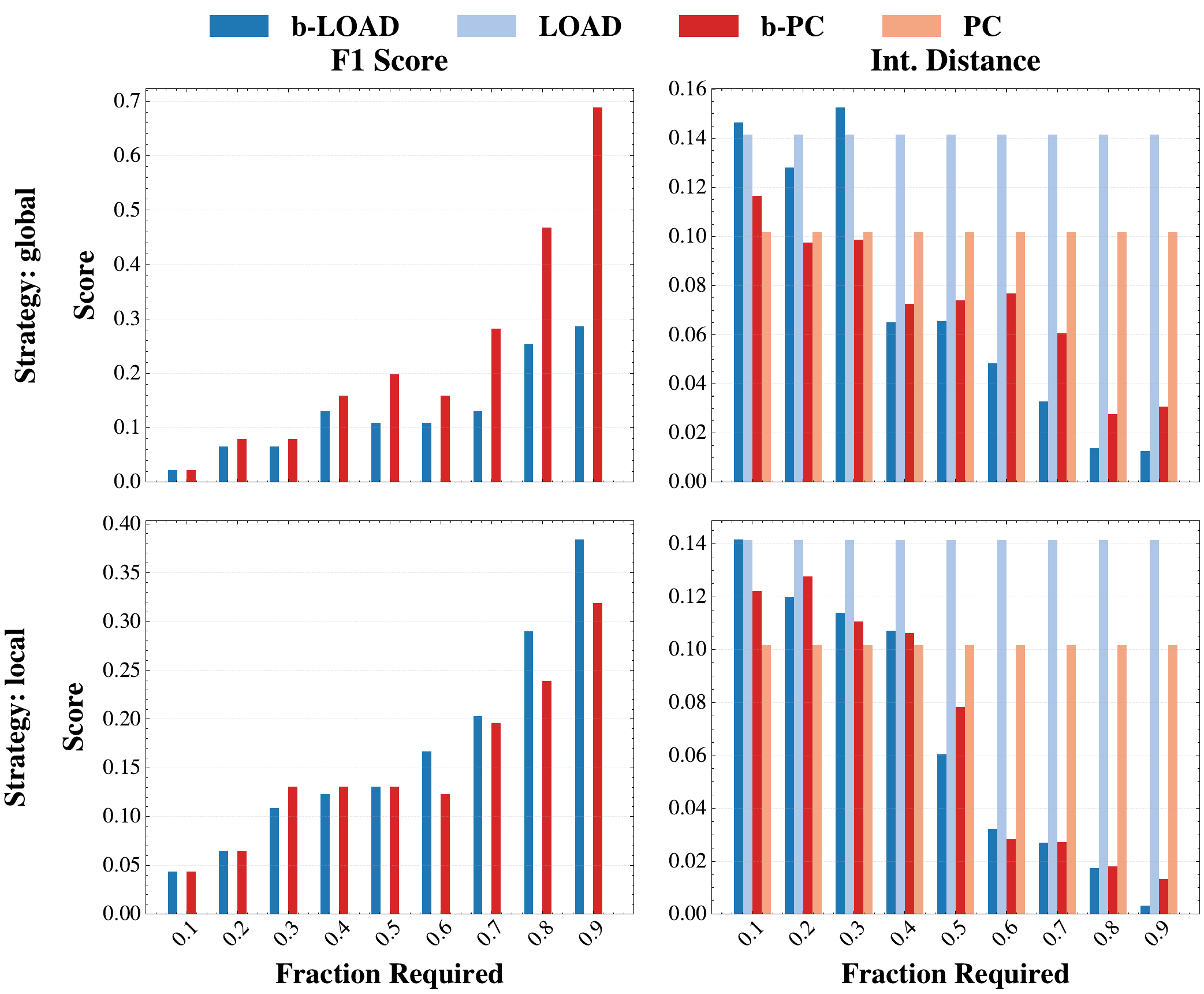}
    \caption{Sachs: F1 (left) and Intervention Distance (right) under
    global (top) and local (bottom) BK sampling. \method\ (blue)
    dominates both metrics.}
    \label{fig:sachs-results}
\end{figure}%

\subsubsection{DREAM4: Results}
\label{subsec:dream4}

In the fragile low-$\Pi$ regime of DREAM4, \Cref{fig:dream4-averaged} details F1 and Intervention Distance across the BK-fraction ($\rho$) axis. Two baseline patterns hold across all fractions: (i) BK-augmented methods entirely dominate knowledge-free baselines (PC and LOAD), which post near-zero F1 throughout, and (ii) \method\ achieves the lowest Intervention Distance at every $\rho$ under both sampling strategies.

The global-vs-local nuance is starkest on Intervention Distance. Under \emph{global} sampling, b-PC's error actually rises at high $\rho$, echoing the Sachs degradation. Conversely, \emph{local} sampling yields predictable, monotonically decreasing errors for both \method\ and b-PC. Locally targeted prior knowledge thus remains the most reliable strategy for downstream causal-effect estimation.

On F1, the optimal method splits by strategy. \method\ strictly dominates under \emph{global} sampling (reaching F1 $\approx 0.23$ vs.\ b-PC's $\approx 0.09$ at $\rho{=}0.9$). Under \emph{local} sampling, b-PC becomes highly competitive and slightly overtakes \method\ at $\rho{=}0.9$ ($\approx 0.27$ vs.\ $\approx 0.23$). Crucially, even when \method\ is marginally penalized on F1 locally, its downstream Intervention Distance remains the lowest, meaning this F1 swap does not translate into worse causal-effect estimates. We unpack the mechanics behind these local-vs-global dynamics and the F1-versus-Intervention-Distance disconnect in \Cref{sec:discussion}.

\begin{figure}[t]
    \centering
    \includegraphics[width=\columnwidth]{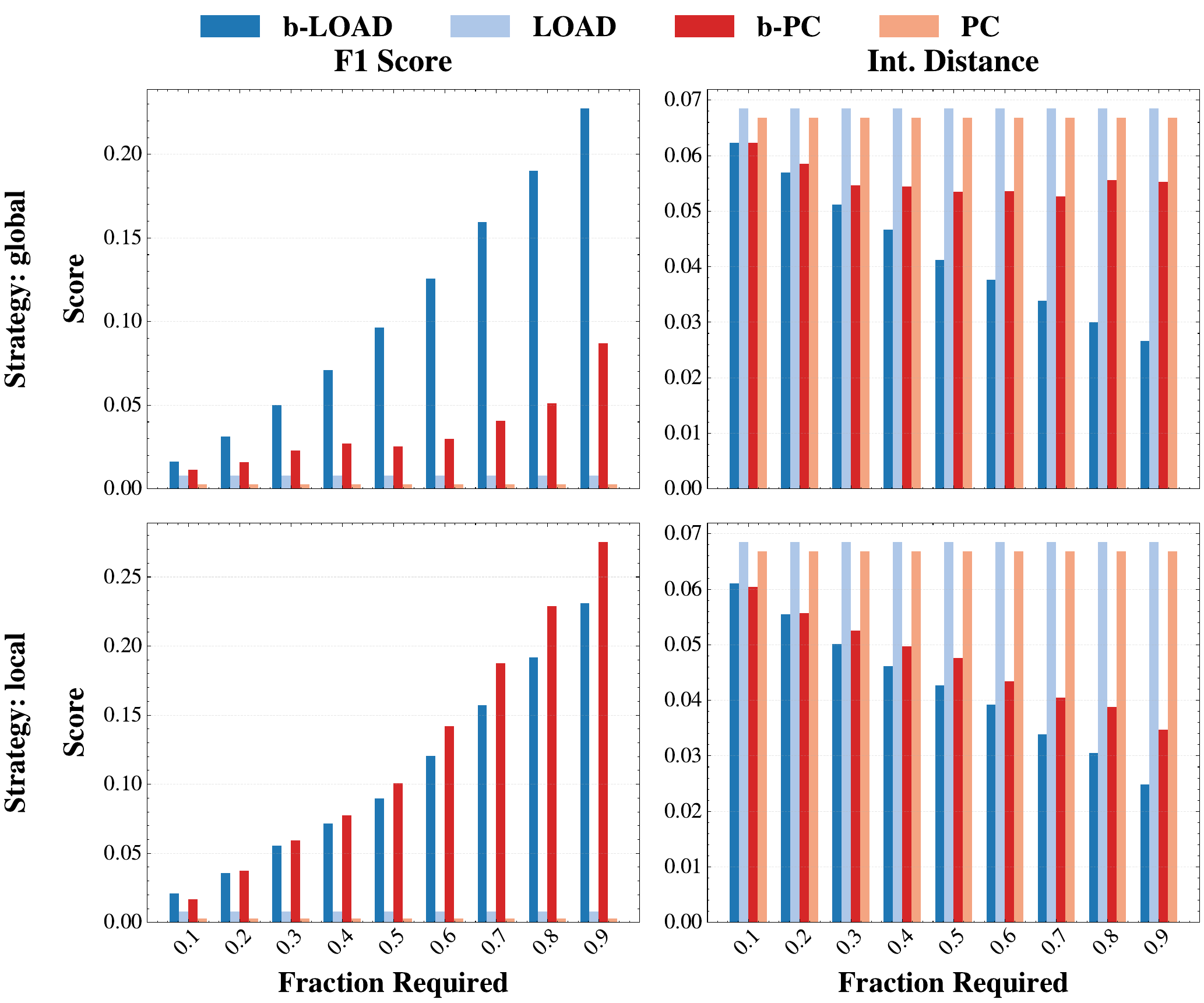}
    \caption{DREAM4 \emph{In~Silico}, averaged over five subnetworks:
    F1 (left) and Intervention Distance (right) under global (top)
    and local (bottom) BK sampling. \method\ lowest Intervention
    Distance throughout; b-PC competitive on F1 only under local
    sampling at high~$\rho$.}
    \label{fig:dream4-averaged}
\end{figure}%

\section{Discussion}\label{sec:discussion}

\myparagraph{When does background knowledge add the most value?}
\method\ yields the highest returns under data scarcity and high causal complexity (low~$\Pi$), where unreliable conditional independence tests leave the local CPDAG weakly oriented and background knowledge bypasses this limitation by deterministically reducing the equivalence class. While the choice between \emph{local} and \emph{global} sampling is metric-dependent—with global sampling often maximizing adjustment-set F1 (\cref{fig:load_comparison,tab:regime_results})—\emph{local} sampling proves remarkably efficient and competitive on Intervention Distance specifically when background knowledge exposure is low (10\%--30\%) (\cref{tab:regime_results,fig:sachs-results,fig:dream4-averaged}). This low-exposure regime is highly realistic, as practitioners in complex expert domains typically possess only partial, high-confidence knowledge localized to specific pathways rather than comprehensive graph coverage. In this realistic low-exposure regime, either sampling strategy of \method\ outperforms the naive global augmentation b-PC across both metrics; at higher exposures the ranking becomes metric-dependent (\Cref{tab:regime_results}). This yields a goal-driven recommendation: globally distributed orientations help when composing adjustment sets, whereas locally targeted orientations are preferable for downstream causal-effect estimation, a distinction particularly relevant in biomedicine where curated pathways are locally abundant.

\myparagraph{F1 versus Intervention Distance.}
The empirical evaluations reveal that improvements in structural accuracy (F1) do not monotonically or predictably translate to better causal effect estimation. Under global sampling, this relationship is highly volatile: on the Sachs dataset, b-PC clearly outperforms \method\ on both F1 and Intervention Distance at certain background knowledge exposures, yet the metrics behave erratically for both algorithms as $\rho$ increases. This unpredictability occurs because randomly distributed global constraints interact chaotically with adjustment sets; an injected global edge might improve the overall F1 score while inadvertently opening a backdoor path that severely biases the target's downstream estimate. Crucially, this erratic global regime is also epistemically demanding, as practitioners rarely possess randomly scattered, domain-wide causal links. However, under local sampling---a far more realistic knowledge budget restricting constraints to the target's immediate neighborhood---this unpredictability vanishes. In this regime, \method\ reliably and monotonically minimizes Intervention Distance, demonstrating that targeted local knowledge is essential for stable, bias-reducing effect estimation.

\myparagraph{Robustness to incorrect background knowledge: a preliminary note.}
Our main experiments assume strictly correct background knowledge~\citep{meek1995causal,perkovic2017interpreting}. Because real-world curated graphs often contain errors or context-specific exceptions, the practical question is how much corruption \method\ tolerates before becoming detrimental. We explored this by replacing required edges with false ones at a mean corruption rate $\mu$ on the synthetic sweep (\Cref{app:noised}); the parallel question for forbidden edges $\mathcal{B}^-$ is left for future work. Results suggest \method\ retains its full advantage at $\mu \leq 0.1$ and degrades gracefully on F1 at moderate noise ($\mu \in \{0.2, 0.3\}$) while preserving a residual Intervention-Distance advantage. It crosses an operational boundary near $\mu = 0.5$, where the F1 advantage collapses. These anecdotal but encouraging findings suggest that a full evaluation under misspecified $\mathcal{B}$ is a natural next step.

\myparagraph{Limitations and Future Directions.}
Several limitations suggest natural directions for future work. First, \method\ inherits the assumption of causal sufficiency; extending it to explicitly model latent confounders via MAGs or PAGs~\citep{zhang2008completeness,borboudakis2012incorporating} would broaden its applicability, and would naturally pair with non-adjustment identification approaches such as the general criterion of~\citet{tian2002general}. Second, the strict requirement of correct prior knowledge could be relaxed by developing mechanisms for \emph{weighted} constraints, using confidence scores to modulate influence in settings where priors are systematically biased. Finally, our empirical scope is restricted to linear-Gaussian SEMs and biological networks; evaluating \method\ on non-Gaussian or nonlinear data and on substantially larger graphs in domains like economics or epidemiology would strengthen the empirical case.

\section{Conclusion}\label{sec:conclusion}
We introduced \method, a knowledge-informed extension of the LOAD~\cite{schubert2025local} algorithm for local discovery of optimal adjustment sets. By incorporating background knowledge directly into the local structure-learning loop, \method\ refines the local equivalence class into a knowledge-constrained~\gls{mpdag} that can identify causal effects unreachable by purely data-driven methods. 
Experimental results on synthetic and real-world datasets show that \method\ leads to large gains in data-scarce, structurally complex regimes and enables causal-effect identification in settings where LOAD returns no answer. Further analyses show that locally targeted constraints yield the most reliable downstream effect estimates, offering actionable guidance for practitioners with partial domain knowledge.

\appendix
\crefalias{section}{appendix}
\section{Proofs for Section 4}

\begin{proof}[\textnormal{\textbf{Proof of~\Cref{lem:soundness}}}]
Background knowledge $\mathcal{B}$ is injected into the local graph prior 
to any structural discovery, so all directed edges in $\mathcal{B}^+$ are 
present from initialization and are consistent with $\mathcal{G}^*$ by 
validity of $\mathcal{B}$. The critical design feature of~\Cref{alg:mbbymbmpdag} is its
\emph{dynamic} \textsc{WaitList} (step~(d) of the \textsc{LocalMPDAG} procedure): after each Markov blanket
discovery and Meek propagation step, the frontier is recomputed as all 
nodes connected to $X$ by an undirected path in the current graph that 
have not yet been processed. This ensures that the nodes are newly reachable 
through a background-induced orientation in $\mathcal{B}$ are 
automatically enqueued, preventing the failure mode in which such nodes 
are never examined for v-structures or orientation implications. 

For every node $Z$ processed from the WaitList, the algorithm recovers its Markov blanket, learns the marginal graph over the relevant variables, and propagates all orientation implications through Meek's rules R1--R4. 
Under perfect conditional independence testing, every v-structure and subsequently oriented edge is guaranteed to match $\mathcal{G}^*$. The WaitList terminates when no new nodes are reachable, at which point all relevant unshielded triples have been examined and all implied orientations propagated. By Theorem~1 of \citet{zheng2026local}, this procedure is sound and complete 
over the local neighborhood of $X$, from which the three 
equalities follow directly.
\end{proof}
\begin{proof}[\textnormal{\textbf{Proof of Proposition~\ref{prop:identifiability}}}]
\textit{(i)} By construction, $\mathcal{M}$ is obtained from $\mathcal{C}$ 
by orienting a subset of undirected edges via the constraints in 
$\mathcal{B}^+$ and propagating through Meek's rules R1--R4. Each 
orientation step eliminates from the equivalence class those DAGs in 
which the constrained edge points in the opposite direction. Since 
$\mathcal{B} \neq \emptyset$, at least one edge is oriented, yielding a 
strictly smaller equivalence class. Acyclicity and the absence of new 
v-structures are preserved by the construction of Meek's rules 
\citep{meek1995causal}.

\textit{(ii)} Identifiability of the causal effect of $X$ on $Y$ via 
adjustment requires amenability of the graph w.r.t.\ $(X, Y)$: all 
proper causal paths from $X$ to $Y$ must begin with a directed edge out 
of $X$ \citep{perkovic2018complete}. If $\mathcal{C}$ is amenable, then 
all edges on proper causal paths incident to $X$ are already directed in 
$\mathcal{C}$; these directions are preserved in $\mathcal{M}$ by 
part~(i), so amenability is maintained. For the converse, consider a 
configuration in which a path $X \,{-}\, W \to Y$ exists in $\mathcal{C}$ 
with the edge $X \,{-}\, W$ undirected; $\mathcal{C}$ is not amenable 
w.r.t.\ $(X, Y)$ since the causal path through $W$ is unresolved. If 
$\mathcal{B}^+$ contains $X \to W$, then $\mathcal{M}$ orients this edge, 
rendering the path directed and $\mathcal{M}$ amenable w.r.t.\ $(X, Y)$, so the effect becomes identifiable.

\textit{(iii)} Validity of $\mathbf{O}^\star$ in $\mathcal{M}$ follows from
part~(i): $\mathbf{O}^\star$ satisfies the generalized adjustment criterion
\citep{perkovic2018complete} in $\mathcal{C}$, and since $\mathcal{M}$
represents a subset of the same DAGs, $\mathbf{O}^\star$ remains valid in
$\mathcal{M}$. The potential for a strictly better set in $\mathcal{M}$
follows from the optimality characterization of \citet{henckel2022graphical}: the optimal set depends on the graph's directed structure, and additional
orientations in $\mathcal{M}$ can change the set of forbidden nodes
(descendants of $Y$ that must be excluded) or the parents of $X$ available for adjustment.
\end{proof}

\section{Preliminary Noise-Robustness Analysis}
\label{app:noised}

Corruption mechanism: a fraction of required edges in $\mathcal{B}^+$
is replaced by false edges at mean rate $\mu$ (Gaussian with
$\sigma = \mu/3$; Bernoulli when the local BK budget is a single edge);
$1000$ seeds per configuration. \Cref{fig:noise} reports the two extreme
regimes ($\mu \in \{0.1, 0.5\}$); intermediate levels interpolate smoothly.
See \Cref{sec:discussion} for the regime breakdown and operating boundaries.

\begin{figure}[t]
    \centering
    \includegraphics[width=\columnwidth]{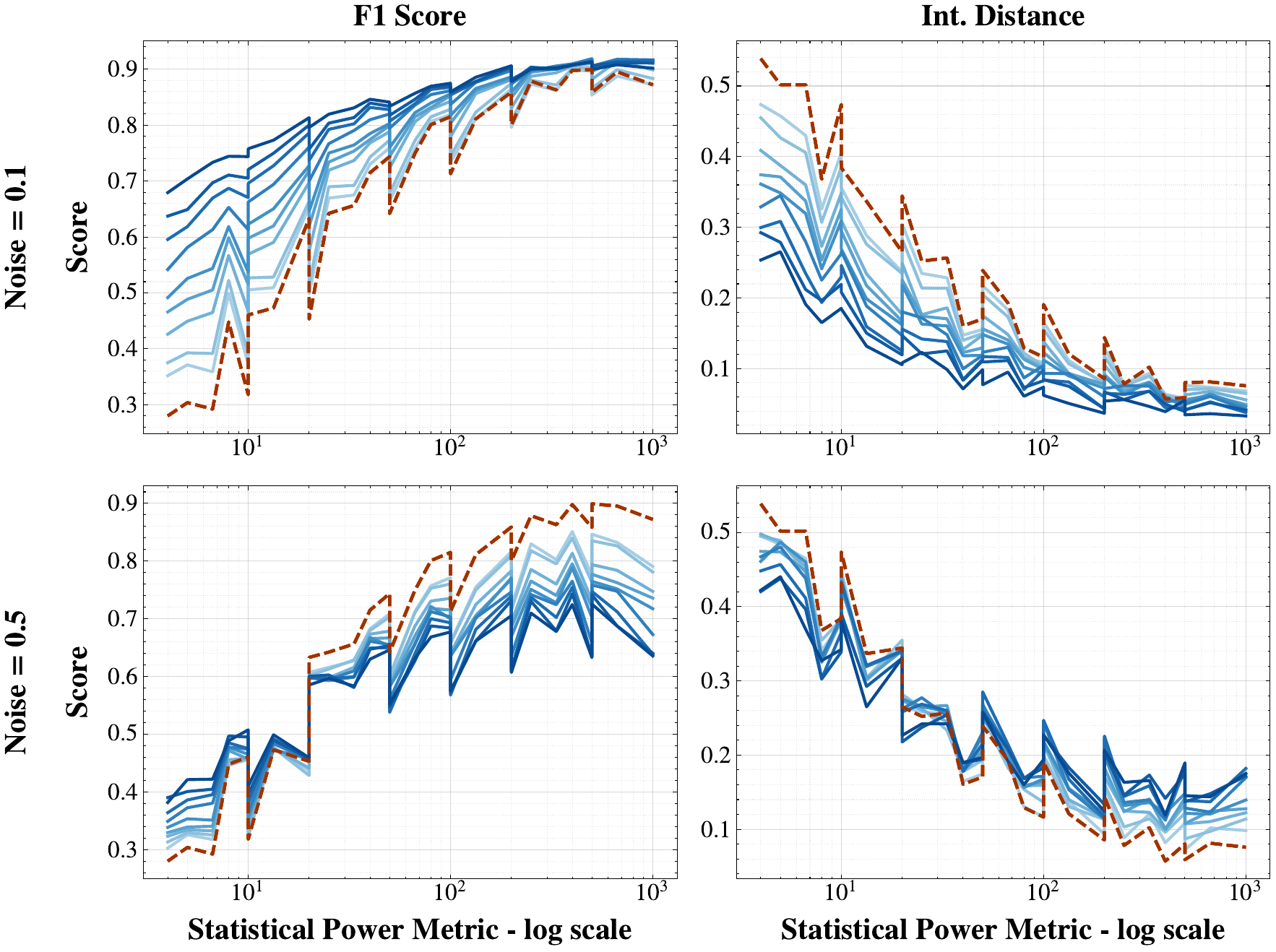}
    \caption{Noise extremes ($\mu{=}0.1$ top, $\mu{=}0.5$ bottom)
    on the synthetic sweep. F1 (left) and Intervention Distance (right)
    vs.\ $\Pi$; line shade encodes $\rho$. Blue: \method\ with noised
    $\mathcal{B}^+$; orange: LOAD.}
    \label{fig:noise}
    \vspace{-1.4em}
\end{figure}%

\section{Experimental Details}
The computational experiments were executed on a shared compute cluster utilizing ThinkSystem SD530 nodes, each equipped with dual 20-core Intel Xeon Gold 6230 processors. Resource allocation was maintained at 4 GB of memory per requested CPU core. Workloads were extensively parallelized across the cluster using SLURM job arrays \cite{Yooetal2003}. Finally, we note that synthetic evaluations on 100-node graphs at higher edge densities (ER4, ER5, and ER6) were excluded due to prohibitive time and memory constraints.

\clearpage
\bibliographystyle{IEEEtranN}
\bibliography{references}

\end{document}